\documentclass[11pt,onecolumn]{IEEEtran}
\usepackage{amsthm}
\usepackage{times,amssymb,amsmath,amsfonts,float,nicefrac,color,bbm,mathrsfs,caption,float}
\usepackage{algorithm,enumerate,multirow,caption,tikz,graphicx}
\usepackage{algpseudocode}
\usepackage[noadjust]{cite}
\usepackage{subcaption}
\usepackage{slashbox}
\usepackage[top=1in,bottom=1in,left=1in,right=1in]{geometry}
\allowdisplaybreaks

\newcommand{\RN}[1]{%
  \textup{\expandafter{\romannumeral#1}}%
}

\tikzset{
  block/.style    = {draw, thick, rectangle, minimum width = 3em},
  sblock/.style      = {draw, thick, rectangle, minimum height = 3em,
    minimum width = 3em}, 
}

\newcommand\remove[1]{}

\allowdisplaybreaks

\newtheorem{theorem}{Theorem}
\newtheorem{definition}{Definition}
\newtheorem{proposition}{Proposition}

\newtheorem{claim}{Claim}

\newtheorem{corollary}{Corollary}

\newtheorem{remark}{Remark}

\newtheorem{cnstr}{Construction}

\newcommand{\cF}{\mathcal{F}}

\DeclareMathOperator{\tot}{tot}

\DeclareMathOperator{\Exp}{Exp}

\begin{document}
\title{Communication-Computation Efficient Gradient Coding}

\author{Min Ye \and \hspace*{.4in} Emmanuel Abbe}

\maketitle
{\renewcommand{\thefootnote}{}\footnotetext{

\vspace{-.2in}
 
\noindent\rule{1.5in}{.4pt}

M. Ye is with Department of Electrical Engineering, Princeton University, Princeton, NJ, email: yeemmi@gmail.com. E. Abbe is with the Program in Applied and Computational Mathematics and Department of Electrical Engineering, Princeton University, and the School of Mathematics, Institute for Advanced Study, Princeton, NJ 08544, USA, email: eabbe@princeton.edu. This work was partly supported by NSF CAREER Award CCF-1552131, NSF Center for the Science of Information CCF-0939370, and the Google Faculty Research Award.
}

\renewcommand{\thefootnote}{\arabic{footnote}}
\setcounter{footnote}{0}

\begin{abstract}
This paper develops coding techniques to reduce the running time of distributed learning tasks. It characterizes the fundamental tradeoff to compute gradients (and more generally vector summations) in terms of three parameters: computation load, straggler tolerance and communication cost.
It further gives an explicit coding scheme that achieves the optimal tradeoff based on recursive polynomial constructions, coding both across data subsets and vector components.   
As a result, the proposed scheme allows to minimize the running time for gradient computations.
Implementations are made on Amazon EC2 clusters using Python with mpi4py package. Results show that the proposed scheme maintains the same generalization error while reducing the running time by $32\%$ compared to uncoded schemes and $23\%$ compared to prior coded schemes focusing only on stragglers (Tandon et al., ICML 2017).
\end{abstract}

\section{introduction}
Distributed computation plays a key role in the computational challenges faced by machine learning for large data sets \cite{Dean12,Abadi16}. This requires overcoming a few obstacles: First the straggler effect, i.e., slow workers that hamper the computation time. Second, the communication cost; gradients in deep learning typically consist nowadays in millions of real-valued components, and the transmission of these high-dimensional vectors can amortize the savings of the computation time in large-scale distributed systems \cite{Recht11,Li14scaling,Li14}.
This has driven researchers to use in particular gradient sparsification and gradient quantization to reduce communication cost \cite{Gupta15,Alistarh17,Wen17}.

More recently, coding theory has found its way into distributed computing \cite{Li15,Lee16,Tandon17,Halbawi17,Raviv17,Dutta16,Dutta17,Yu17,Yu17FT,Yu18,
Yang17,Li17,Karakus17,Charles17,Zhu17,Li17fund}, following the path of exporting coding techniques to distributed storage \cite{Dimakis10,Ye17}, caching \cite{Maddah15} and queuing \cite{Joshi15}. A few works have also initiated the use of coding techniques in distributed learning \cite{Li15,Lee16,Tandon17}. Of particular interest to us is \cite{Tandon17}, which introduces coding techniques to mitigate the effect of stragglers in gradient computation. While this is a central task in machine learning, \cite{Tandon17} does not take into account the communication cost which is important in such applications as mentioned above.

This paper takes a global view on the running time of distributed learning tasks by considering the three parameters, namely, computation load, straggler tolerance and communication cost.
We identify a three-fold fundamental tradeoff between these parameters in order to efficiently compute gradients (and more generally summations of vectors), exploiting distributivity both across data subsets and vector components. 
The tradeoff reads 
\begin{equation}\label{eq:1mrt}
\frac{d}{k}\ge \frac{s+m}{n},
\end{equation}
where $n$ is the number of workers, $k$ is the number of data subsets, $d$ is the number of data subsets assigned to each worker, $s$ is the number of stragglers, and $m$ is the communication reduction factor. This generalizes the results in \cite{Tandon17} that correspond to $m=1$. Note that one cannot derive
\eqref{eq:1mrt} from the results of \cite{Tandon17}, and we will explain this in more detail below.

We further give an explicit code construction based on recursive polynomials that achieves the derived tradeoff. 
The key steps in our coding scheme are as follows: In order to reduce the dimension of transmitted vector for each worker, we first partition the coordinates of the gradient vector into $m$ groups of equal size.
Then we design two matrices $B$ and $V$, where the $(n-s)\times n$ matrix $V$ has the property that any $(n-s)\times (n-s)$ submatrix is invertible. This property corresponds to the requirement that our coding scheme can tolerate {\em any} $s$ stragglers, and it can be easily satisfied by setting $V$ to be a (non-square) Vandermonde matrix. Furthermore, the $(mn)\times (n-s)$ matrix $B$ satisfies the following two property: (1) the last $m$ columns of $B$ consisting of $n$ identity matrices of size $m\times m$; 
(2) for every $j\in[n]$, the product of the $i$th row of $B$ and the $j$th column of $V$ must be $0$ for a specific set of values of $i$, and the cardinality of this set is $(n-d)m$. The first property of $B$ guarantees the recovery of the sum gradient vector, and the second property ensures that each worker is assigned at most $d$ data subsets.
We make use of the natural connection between the Vandermonde structure and polynomials to construct our matrix $B$ recursively: More precisely, we can view each row of $B$ as coefficients of some polynomial, and the product of $B$ and $V$ simply consists of the evaluations of these polynomials at certain points. We can then define these polynomials by specifying their roots so that the two properties of $B$ are satisfied.
We also mention that the conditions in our construction are more restrictive than those in \cite{Dutta16} and \cite{Tandon17,Halbawi17,Raviv17}: In our setting, the conditions in \cite{Dutta16} only require that the last $m$ columns of $B$ contain at most $n$ nonzero entries, and no requirements are imposed on the positions of these nonzero entries; as mentioned above, \cite{Tandon17,Halbawi17,Raviv17} only deal with the special case of $m=1$ and do not allow for dimensionality reduction of the gradient vectors. Due to these more relaxed conditions, the constructions in \cite{Dutta16} and \cite{Tandon17,Halbawi17,Raviv17} do not have the recursive polynomial structure, which is the main technical novelty in our paper.

We further take numerical stability issue into consideration, and characterize an achievable region of the triple $(d,s,m)$ under a given upper bound $\kappa$ of condition numbers of all the operations in the gradient reconstruction phase. We also present another coding scheme based on random matrices to achieve this region.

We support our theoretical findings by implementing our scheme on Amazon EC2 clusters using Python with mpi4py package. Experimental results show that the proposed scheme reduces the running time by $32\%$ compared to uncoded schemes and by $23\%$ compared to prior work \cite{Tandon17}, while maintaining the same generalization error on the Amazon Employee Access dataset from Kaggle, which was also used in \cite{Tandon17} for state-of-the-art experiments.

\subsection{Related literature}
Slow workers (processors) called ``stragglers" can hamper the computation time as the taskmaster needs to wait for all workers to complete their processing.
Recent literature proposes adding redundancy in computation tasks of each worker so that the taskmaster can compute the final result using outputs from only a subset of workers and ignore the stragglers.
The most popular ways to introduce redundancy in computation are based on either replication schemes or coding theoretic techniques \cite{Ana13,Wang14,Shah16,Lee16}.
Lee et al. \cite{Lee16} initialized the study of using erasure-correcting codes to mitigate straggler effects for linear machine learning tasks such as linear regression and matrix multiplication. 
Subsequently, Dutta et al. proposed new efficient coding schemes to calculate convolutions \cite{Dutta17} and the product of a matrix and a long vector \cite{Dutta16}, Yu et al.  introduced optimal coding schemes to compute high-dimensional matrix multiplication \cite{Yu17,Yu18} and Fourier Transform \cite{Yu17FT}, and Yang et al. developed coding methods for parallel iterative linear solver \cite{Yang17}.
Tandon et al. \cite{Tandon17} further used coding theoretic methods to avoid stragglers in nonlinear learning tasks. More specifically, \cite{Tandon17} presented an optimal trade-off between the computation load and {\em straggler tolerance} (the number of tolerable stragglers) in synchronous gradient descent for {\em any} loss function.
Several code constructions achieving this trade-off were given in \cite{Tandon17,Halbawi17,Raviv17}.
Li et al. \cite{Li17} considered distributed gradient descent under a probabilistic model and proposed the Batched Coupon's Collector scheme to alleviate straggler effect under this model.
At the same time, the schemes in \cite{Tandon17,Halbawi17,Raviv17} are designed to combat stragglers for the worst-case scenario.
While most research focused on recovering the exact results in the presence of stragglers,
 \cite{Karakus17,Raviv17,Charles17} suggested allowing some small deviations from the exact gradient in each iteration of the gradient descent and showed that one can obtain a good approximation of the original solution by using coding theoretic methods.
Very recently, Zhu et al. \cite{Zhu17} proposed a sequential approximation method for distributed learning in the presence of stragglers, and their method is also based on erasure-correcting codes.

As mentioned above, high network communication cost for synchronizing gradients and parameters
is also a well-known bottleneck of distributed learning. In particular for deep learning, gradient vectors typically consist of millions of real numbers, and for large-scale distributed systems, transmissions of high-dimensional gradient vectors might even amortize the savings of computation time \cite{Recht11,Li14scaling,Li14}.
The most widely used methods to reduce communication cost in the literature are based on gradient sparsification and gradient quantization \cite{Gupta15,Alistarh17,Wen17}.

In this paper we directly incorporate the communication cost into the framework of reducing running time for gradient computation, in addition to computation load and straggler tolerance. In particular, we take advantage of distributing the computations over subsets of vector components in addition to subsets of data samples. The advantages of our coding scheme over the uncoded schemes and the schemes in \cite{Tandon17,Halbawi17,Raviv17} are demonstrated by both experimental results and numerical analysis in Sections \ref{exp1} and \ref{ana1}. 
We also strengthen the numerical analyses by studying the behavior of the running time using probabilistic models for the computation and communication times, obtaining improvements that are consistent with the outcome of the Amazon experiments (see Section \ref{ana1}). 
Our results apply to both batch gradient descent and mini-batch stochastic gradient descent (SGD), which is the most popular algorithm in large-scale distributed learning.
Moreover, our coding theoretic method is orthogonal to the gradient sparsification and gradient quantization methods \cite{Gupta15,Alistarh17,Wen17}. In other words, our method can be used on top of the latter ones.

As a final remark, \cite{Li15,Li17fund} also studied the trade-off between computation and communication in distributed learning, but the problem setup in \cite{Li15,Li17fund} is different from our work in nature. We study distributed gradient descent while \cite{Li15,Li17fund} focused on MapReduce framework. The communication in our problem is from all the worker nodes to one master node, while the communication in \cite{Li15,Li17fund} is from all workers to all workers, and there is no master node in \cite{Li15,Li17fund}. 
This difference in problem setup leads to completely different results and techniques.


\section{Problem formulation and main results}
We begin with a brief introduction on distributed gradient descent. Given a dataset $D=\{(x_i,y_i)\}_{i=1}^N$, where $x_i\in \mathbb{R}^l$ and $y\in\mathbb{R}$, we want to learn parameters $\beta\in\mathbb{R}^l$ by minimizing a generic loss function $L(D;\beta):=\sum_{i=1}^N L(x_i,y_i;\beta)$, for which gradient descent is commonly used. More specifically, we begin with some initial guess of $\beta$ as $\beta^{(0)}$, and then update the parameters according to the following rule:
\begin{equation}\label{eq:upd}
\beta^{(t+1)} = h(\beta^{(t)},g^{(t)}),
\end{equation}
where $g^{(t)}:=\nabla L(D;\beta^{(t)}) = \sum_{i=1}^N \nabla L(x_i,y_i;\beta^{(t)})$ is the gradient of the loss at the current estimate of the parameters and $h$ is a gradient-based optimizer.
As in \cite{Tandon17},
we assume that there are $n$ workers $W_1,W_2,\dots,W_n$, and that the original dataset $D$ is partitioned into $k$ subsets of equal size, denoted as $D_1,D_2,\dots,D_k$.
Define the partial gradient vector of $D_i$ as $g_i^{(t)}:=\sum_{(x,y)\in D_i}\nabla L(x,y;\beta^{(t)})$. Clearly $g^{(t)}=g_1^{(t)}+g_2^{(t)}+\dots+g_k^{(t)}$. 
 Suppose that each worker is assigned $d$ data subsets, and there are $s$ stragglers, i.e., we only wait for the results from the first $n-s$ workers.
For $i=1,2,\dots,n$, we write the datasets assigned to worker $W_i$ as $\{D_{i_1},D_{i_2},\dots,D_{i_d}\}$.
Each worker computes its partial gradient vectors $g_{i_1}^{(t)},g_{i_2}^{(t)},\dots,g_{i_d}^{(t)}$ and returns $f_i(g_{i_1}^{(t)},g_{i_2}^{(t)},\dots,g_{i_d}^{(t)})$, a prespecified function of these partial gradients.
In order to update the parameters according to \eqref{eq:upd},
we require that  the sum gradient vector $g^{(t)}$ can be recovered from the results of the first $n-s$ workers no matter who the $s$ stragglers will be. 
Due to complexity consideration, we would further like $f_i,i\in[n]$ to be linear functions.
 Lee et al. \cite{Tandon17} showed that this is possible if and only if 
$$
\frac{d}{k}\ge \frac{s+1}{n}.
$$
Since the functions $f_i,i\in[n]$ are time invariant, in the rest of this paper we will omit the superscript $(t)$ for simplicity of notation.
Recall that in batch gradient descent, we use all the samples to update parameters in each iteration, and in mini-batch SGD we use a small portion of the whole dataset in each iteration.
Since we only focus on each iteration of the gradient descent algorithm, our results apply to both batch gradient descent and mini-batch SGD.

Let us write each partial gradient vector as
$g_i=(g_i(0),g_i(1),\dots,g_i(l-1))$ for $i=1,2,\dots,k$.
We will show that when $s\le\frac{d}{k}n-1$, each worker only needs to transmit a vector\footnote{Assume that $k|(dn)$ and $(\frac{d}{k}n-s)|l$.} of dimension
$l/(\frac{d}{k}n-s)$. In other words, we can reduce the communication cost by a factor of $\frac{d}{k}n-s$.

Roughly speaking, \cite{Tandon17} showed the following two-dimensional tradeoff: if we assign more computation load at each worker, then we can tolerate more stragglers. In this paper we will show a three-dimensional tradeoff between computation load at each worker, straggler tolerance and the communication cost: for a fixed computation load, we can reduce the communication cost by waiting for results from more workers.
Fig.~\ref{fig:sim} uses a toy example to illustrate this tradeoff as well as the basic idea of how to reduce the communication cost. In Fig.~\ref{fig:sim} the gradient vector has dimension $l=2$, and it is clear that this idea extends to gradient vectors of any dimension (by padding a zero when $l$ is odd).
To quantify the tradeoff, we introduce the following definition.

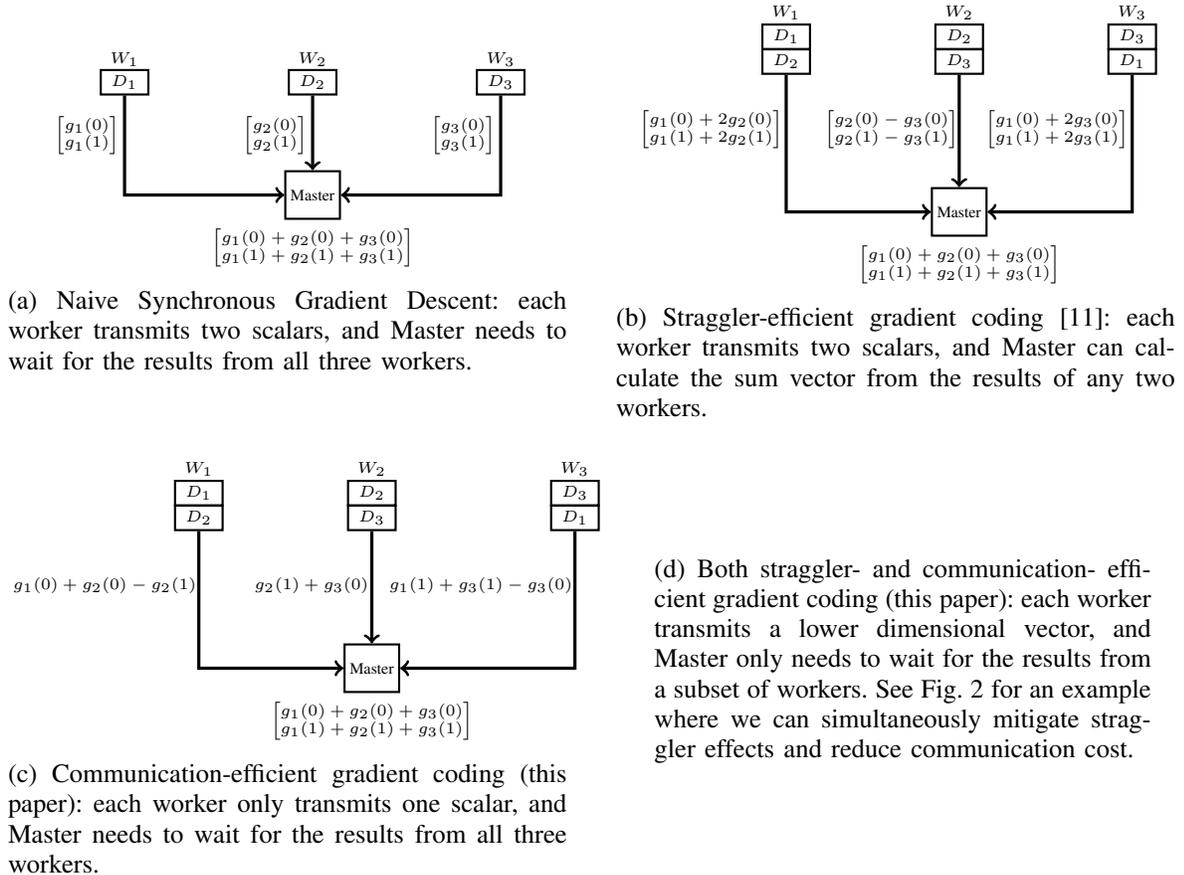
\begin{figure}[!htp]
{\tiny
\begin{subfigure}{.45\linewidth}
\centering
\begin{tikzpicture}
\draw
 node at (0,0) []   {$W_1$}
 node at (2.5,0) []  {$W_2$}
 node at (5,0) []  {$W_3$}
 node at (0,-0.3) [block] (w1)  {$D_1$}
 node at (2.5,-0.3) [block] (w2)  {$D_2$}
 node at (5,-0.3) [block] (w3)  {$D_3$}
node at (2.5,-1.8) [sblock] (m) {Master}
node at (-0.5,-1) [] {$\begin{bmatrix} g_1(0) \\ g_1(1) \end{bmatrix}$}
node at (2,-1) [] {$\begin{bmatrix} g_2(0) \\ g_2(1) \end{bmatrix}$}
node at (4.5,-1) [] {$\begin{bmatrix} g_3(0) \\ g_3(1) \end{bmatrix}$}
node at (2.5,-2.5) [] {$\begin{bmatrix} g_1(0)+g_2(0)+g_3(0) \\ g_1(1)+g_2(1)+g_3(1)
\end{bmatrix}$};
\draw[very thick,->](w1) |- node {}(m);
\draw[very thick,->](w2) -- node {}(m);
\draw[very thick,->](w3) |- node {}(m);
\end{tikzpicture}
\caption{Naive Synchronous Gradient Descent: each worker transmits two scalars, and Master needs to wait for the results from all three workers.}
\end{subfigure} \hspace*{0.2in}
\begin{subfigure}{.45\linewidth}
\centering
\begin{tikzpicture}
\draw
 node at (0,0.65) []   {$W_1$}
 node at (2.3,0.65) []  {$W_2$}
 node at (4.6,0.65) []  {$W_3$}
 node at (0,0) [block] (w1) {$D_2$}
 node at (0,0.33) [block]   {$D_1$}
 node at (2.3,0) [block] (w2) {$D_3$}
 node at (2.3,0.33) [block]   {$D_2$}
 node at (4.6,0) [block] (w3) {$D_1$}
 node at (4.6,0.33) [block]   {$D_3$}
node at (2.3,-2) [sblock] (m) {Master}
node at (-1,-0.9) [] {$\begin{bmatrix} g_1(0)+2g_2(0) \\ g_1(1)+2g_2(1) \end{bmatrix}$}
node at (1.4,-0.9) [align=left] {$\begin{bmatrix} g_2(0) - g_3(0) \\ g_2(1) - g_3(1) \end{bmatrix}$}
node at (3.6,-0.9) [align=left] {$\begin{bmatrix} g_1(0)+2g_3(0) \\ g_1(1)+2g_3(1) \end{bmatrix}$}
node at (2.3,-2.7) [align=left] {$\begin{bmatrix} g_1(0)+g_2(0)+g_3(0) \\ g_1(1)+g_2(1)+g_3(1) \end{bmatrix}$};
\draw[very thick,->](w1) |- node {}(m);
\draw[very thick,->](w2) -- node {}(m);
\draw[very thick,->](w3) |- node {}(m);
\end{tikzpicture}
\caption{Straggler-efficient gradient coding \cite{Tandon17}: each worker transmits two scalars, and Master can calculate the sum vector from the results of any two workers.}
\end{subfigure} 
\par\bigskip

\begin{subfigure}{.45\linewidth}
\centering
\begin{tikzpicture}
\draw
 node at (0,0.65) []   {$W_1$}
 node at (2.3,0.65) []  {$W_2$}
 node at (5,0.65) []  {$W_3$}
 node at (0,0) [block] (w1) {$D_2$}
 node at (0,0.33) [block]   {$D_1$}
 node at (2.3,0) [block] (w2) {$D_3$}
 node at (2.3,0.33) [block]   {$D_2$}
 node at (5,0) [block] (w3) {$D_1$}
 node at (5,0.33) [block]   {$D_3$}
node at (2.3,-2) [sblock] (m) {Master}
node at (-1.25,-0.9) [] {$g_1(0)+g_2(0)-g_2(1)$}
node at (1.5,-0.9) [] {$g_2(1) + g_3(0)$}
node at (3.75,-0.9) [] {$g_1(1)+g_3(1)-g_3(0)$}
node at (2.3,-2.7) [align=left] {$\begin{bmatrix} g_1(0)+g_2(0)+g_3(0) \\ g_1(1)+g_2(1)+g_3(1) \end{bmatrix}$};
\draw[very thick,->](w1) |- node {}(m);
\draw[very thick,->](w2) -- node {}(m);
\draw[very thick,->](w3) |- node {}(m);
\end{tikzpicture}
\caption{Communication-efficient gradient coding (this paper): each worker only transmits one scalar, and Master needs to wait for the results from all three workers.}
\end{subfigure} \hspace*{0.4in}
\begin{subfigure}{.4\linewidth}
\caption{Both straggler- and communication- efficient gradient coding (this paper): each worker  transmits a lower dimensional vector, and Master only needs to wait for the results from a subset of workers.
See Fig.~\ref{fig:cmp} for an example where we can simultaneously mitigate straggler effects and reduce communication cost.}
\end{subfigure}
\caption{The idea of communication efficient gradient coding.}
\label{fig:sim}}
\end{figure}

\begin{definition} \label{def:m}
Given $n$ and $k$, we say that a triple of nonnegative integers $(d,s,m)$ satisfying that $1\le d\le k$ and $m\ge 1$ is achievable\footnote{Throughout we assume that $m|l$. Since $l$ is typically very large and $m$ is relatively small, the condition $m|l$ can always be satisfied by padding a few zeroes at the end of the gradient vectors.} if there is a distributed synchronous gradient descent scheme such that
\begin{enumerate}
\item Each worker is assigned $d$ data subsets. 
\item There are $n$ functions $f_1,f_2,\dots,f_n$ from $\mathbb{R}^{dl}$ to $\mathbb{R}^{l/m}$ such that the gradient vector $g_1+g_2+\dots+g_k$ can be recovered from any $n-s$ out of the following $n$ vectors
\begin{equation}\label{eq:fi}
f_i(g_{i_1},g_{i_2},\dots,g_{i_d}),i=1,2,\dots,n,
\end{equation}
where $i_1,i_2,\dots,i_d$ are the indices of datasets assigned to worker $W_i$.
\item $f_1,f_2,\dots,f_n$ are linear functions. In other words, $f_i(g_{i_1},g_{i_2},\dots,g_{i_d})$ is a linear combination of the coordinates of the partial gradient vectors $g_{i_1},g_{i_2},\dots,g_{i_d}$.
\end{enumerate}
\end{definition}

For readers' convenience, we list the main notation in Table~\ref{tb:nt}. Next we state the main theorem of this paper.
\begin{table}[t]
\centering
\begin{tabular}{| c | c |}
\hline 
$n$ & the number of workers \\  \hline
$k$ & the number of data subsets in total; in most part of the paper we assume $n=k$ \\  \hline
$d$ & the number of data subsets assigned to each worker \\  \hline
$s$ & the number of stragglers \\  \hline
$m$ & the communication cost reduction factor \\  \hline
$l$ & the dimension of gradient vectors \\  \hline
$g_i,i\in[k]$ & the partial gradient vector of data subset $D_i$;
$g_i=(g_i(0),g_i(1),\dots,g_i(l-1))$ \\  \hline
$f_i(g_{i_1},g_{i_2},\dots,g_{i_d})$ & the transmitted vector of worker $W_i$, sometimes abbreviated as $f_i$ \\ \hline
\end{tabular}
\caption{Main notation}
\label{tb:nt}
\end{table}

\begin{theorem}\label{thm:main}
Let $k,n$ be positive integers.
A triple $(d,s,m)$ is achievable if and only if 
\begin{equation}\label{eq:obj}
\frac{d}{k}\ge \frac{s+m}{n}.
\end{equation}
\end{theorem}
The converse proof is given in Appendix~\ref{ap:conv}, and the achievability scheme is given in Section~\ref{sect:code}.

Note that the special case $m=1$ in Theorem~\ref{thm:main} is the same as the case considered in \cite{Tandon17,Halbawi17,Raviv17}. We also remark that although \eqref{eq:obj} looks very similar to Theorem 1 in \cite{Dutta16}, their coding scheme can not be used to achieve \eqref{eq:obj} with equality when $m>1$. In Appendix~\ref{ap:dif}, we discuss the differences between our work and \cite{Dutta16} in detail. In particular, we show that the constraint in our problem is stronger than that in \cite{Dutta16}.

\begin{remark}\label{rm:n=k}
Notice that the computation load at each worker is known by $\frac{d}{k}$, not the value of $k$ itself.
We are interested in achieving the optimal computation load in \eqref{eq:obj}, and the value of $k$ does not matter. Therefore we will assume that $k=n$ for the remainder of this paper (except in Appendix~\ref{ap:conv} and Appendix~\ref{ap:dif}).
\end{remark}
Under this assumption, \eqref{eq:obj} has the following simple form
\begin{equation}\label{eq:sim}
d\ge s+m.
\end{equation}

In Fig.~\ref{fig:cmp} we take $n=k=5,d=3,l=2$, and show the implementation for two different choices of the pair $(s,m)$.
The communication cost of Fig.~\ref{sf:2} is half of that of Fig.~\ref{sf:1}, but the system in Fig.~\ref{sf:2} can only tolerate one straggler while the system in Fig.~\ref{sf:1} can tolerate two stragglers.
Table \ref{tb:how} below shows how to calculate the sum gradient vector in Fig.~\ref{sf:2} when there is one straggler. In the table we abbreviate $f_i(g_{i_1},g_{i_2},\dots,g_{i_d})$ in \eqref{eq:fi} as $f_i$, i.e., $f_i$ is the transmitted vector of $W_i$.

\begin{figure}[!htp]
{\tiny
\begin{subfigure}{\linewidth}
\centering
\begin{tikzpicture}
\draw
 node at (-3.5,0.94) []   {$W_1$}
 node at (0,0.94) []   {$W_2$}
 node at (3.5,0.94) []  {$W_3$}
 node at (7,0.94) []  {$W_4$}
 node at (10.5,0.94) []   {$W_5$}
 node at (-3.5,0) [block] (w1) {$D_3$}
 node at (-3.5,0.32) [block]   {$D_2$}
 node at (-3.5,0.64) [block]   {$D_1$}
 node at (0,0) [block] (w2) {$D_4$}
 node at (0,0.32) [block]   {$D_3$}
 node at (0,0.64) [block]   {$D_2$}
 node at (3.5,0) [block] (w3) {$D_5$}
 node at (3.5,0.32) [block]   {$D_4$}
 node at (3.5,0.64) [block]   {$D_3$}
 node at (7,0) [block] (w4) {$D_1$}
 node at (7,0.32) [block]   {$D_5$}
 node at (7,0.64) [block]   {$D_4$}
 node at (10.5,0) [block] (w5) {$D_2$}
 node at (10.5,0.32) [block]   {$D_1$}
 node at (10.5,0.64) [block]   {$D_5$}
node at (3.5,-1.5) [sblock] (m) {Master}
node at (-5,-0.7) [] {$\begin{bmatrix} g_1(0)+3g_2(0)+6g_3(0) \\ g_1(1)+3g_2(1)+6g_3(1) \end{bmatrix}$}
node at (-1.55,-0.7) [] {$\begin{bmatrix} 2g_2(0) + 6g_3(0) - 3g_4(0) \\ 2g_2(1) + 6g_3(1) - 3g_4(1) \end{bmatrix}$}
node at (2.05,-0.7) [] {$\begin{bmatrix} g_3(0)-2g_4(0)+g_5(0) \\ g_3(1)-2g_4(1)+g_5(1) \end{bmatrix}$}
node at (5.5,-0.7) [] {$\begin{bmatrix} 3g_4(0)-6g_5(0)-2g_1(0) \\ 3g_4(1)-6g_5(1)-2g_1(1) \end{bmatrix}$}
node at (9,-0.7) [] {$\begin{bmatrix} 6g_5(0)+3g_1(0)+g_2(0) \\ 6g_5(1)+3g_1(1)+g_2(1) \end{bmatrix}$}
node at (3.5,-2.1) [] {$\begin{bmatrix} g_1(0)+g_2(0)+g_3(0)+g_4(0)+g_5(0) \\ g_1(1)+g_2(1)+g_3(1)+g_4(1)+g_5(1) \end{bmatrix}$};
\draw[very thick,->](w1) |- node {}(m);
\draw[very thick,->](w2) |- node {}(m);
\draw[very thick,->](w3) -- node {}(m);
\draw[very thick,->](w4) |- node {}(m);
\draw[very thick,->](w5) |- node {}(m);
\end{tikzpicture}
\caption{$s=2,m=1$: Each worker transmits two scalars, and Master can calculate the sum vector from the results of any $3$ workers.}
\label{sf:1}
\end{subfigure}
\begin{subfigure}{\linewidth}
\centering
\begin{tikzpicture}
\draw
 node at (-3.5,0.94) []   {$W_1$}
 node at (0,0.94) []   {$W_2$}
 node at (3.5,0.94) []  {$W_3$}
 node at (7,0.94) []  {$W_4$}
 node at (10.5,0.94) []   {$W_5$}
 node at (-3.5,0) [block] (w1) {$D_3$}
 node at (-3.5,0.32) [block]   {$D_2$}
 node at (-3.5,0.64) [block]   {$D_1$}
 node at (0,0) [block] (w2) {$D_4$}
 node at (0,0.32) [block]   {$D_3$}
 node at (0,0.64) [block]   {$D_2$}
 node at (3.5,0) [block] (w3) {$D_5$}
 node at (3.5,0.32) [block]   {$D_4$}
 node at (3.5,0.64) [block]   {$D_3$}
 node at (7,0) [block] (w4) {$D_1$}
 node at (7,0.32) [block]   {$D_5$}
 node at (7,0.64) [block]   {$D_4$}
 node at (10.5,0) [block] (w5) {$D_2$}
 node at (10.5,0.32) [block]   {$D_1$}
 node at (10.5,0.64) [block]   {$D_5$}
node at (3.5,-1.5) [sblock] (m) {Master}
node at (-5,-0.7) [align=right] {$g_1(0)+3g_2(0)+6g_3(0)$ \\ $-3g_1(1)-3g_2(1)+6g_3(1)$}
node at (-1.45,-0.7) [align=right] {$2g_2(0) + 6g_3(0) - 3g_4(0)$ \\ $+12g_3(1) + 3g_4(1)$}
node at (2.2,-0.7) [align=right] {$g_3(0)-2g_4(0)+g_5(0)$ \\ $+3g_3(1)-3g_5(1)$}
node at (5.55,-0.7) [align=right] {$3g_4(0)-6g_5(0)-2g_1(0)$ \\ $+3g_4(1)+12g_5(1)$}
node at (8.95,-0.7) [align=right] {$6g_5(0)+3g_1(0)+g_2(0)$ \\ $-6g_5(1)+3g_1(1)+3g_2(1)$}
node at (3.5,-2.1) [align=left] {$\begin{bmatrix} g_1(0)+g_2(0)+g_3(0)+g_4(0)+g_5(0) \\ g_1(1)+g_2(1)+g_3(1)+g_4(1)+g_5(1) \end{bmatrix}$};
\draw[very thick,->](w1) |- node {}(m);
\draw[very thick,->](w2) |- node {}(m);
\draw[very thick,->](w3) -- node {}(m);
\draw[very thick,->](w4) |- node {}(m);
\draw[very thick,->](w5) |- node {}(m);
\end{tikzpicture}
\caption{$s=1,m=2$: Each worker only transmits one scalar, and Master can calculate the sum vector from the results of any $4$ workers. Table~\ref{tb:how} shows how to do this calculation.}
\label{sf:2}
\end{subfigure}
\caption{Tradeoff between communication cost and straggler tolerance}
\label{fig:cmp}}
\end{figure}
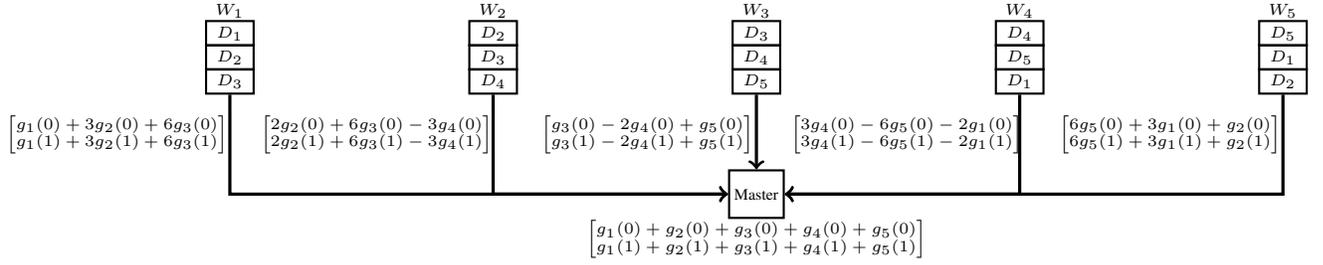
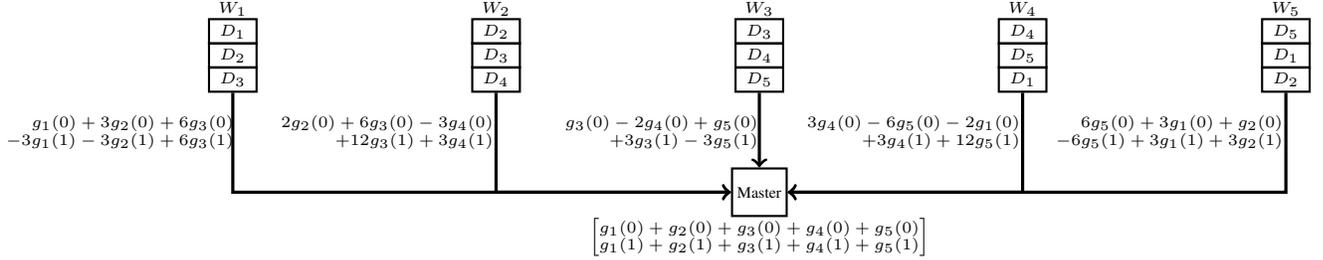

\begin{table}[h]
\centering
\begin{tabular}{| c | c | c |}
\hline 
 Straggler & Calculate $g_1(0)+g_2(0)+g_3(0)+g_4(0)+g_5(0)$ & Calculate $g_1(1)+g_2(1)+g_3(1)+g_4(1)+g_5(1)$ \\ [0.2em] \hline
$W_1$ & $\frac{1}{2}f_2-2f_3-\frac{1}{2}f_4$  &
$-\frac{1}{6}f_2+f_3+\frac{1}{2}f_4+\frac{1}{3}f_5$ \\ [0.2em] \hline
$W_2$ & $\frac{1}{4}f_1-\frac{1}{2}f_3+\frac{1}{4}f_5$  & 
$-\frac{1}{12}f_1+\frac{1}{2}f_3+\frac{1}{3}f_4+\frac{1}{4}f_5$ \\ [0.2em] \hline
$W_3$ & $\frac{1}{3}f_1-\frac{1}{6}f_2+\frac{1}{6}f_4+\frac{1}{3}f_5$  &
$-\frac{1}{6}f_1+\frac{1}{6}f_2+\frac{1}{6}f_4+\frac{1}{6}f_5$ \\ [0.2em] \hline
$W_4$ & $\frac{1}{4}f_1-\frac{1}{2}f_3+\frac{1}{4}f_5$ & 
$-\frac{1}{4}f_1+\frac{1}{3}f_2-\frac{1}{2}f_3+\frac{1}{12}f_5$ \\ [0.2em] \hline
$W_5$ & $\frac{1}{2}f_2-2f_3-\frac{1}{2}f_4$ &
$-\frac{1}{3}f_1+\frac{1}{2}f_2-f_3-\frac{1}{6}f_4$ \\ [0.2em] \hline
\end{tabular}
\caption{Calculate the sum gradient vector in Fig.~\ref{sf:2} when there is one straggler.}
\label{tb:how}
\end{table}

\subsection{Achievable region with stability constraints}
In the proof of Theorem~\ref{thm:main}, we use Vandermonde matrices and assume that all the computations have infinite precision, which is not possible in real world applications. According to our experimental results, the stability issue of Vandermonde matrices can be ignored up to $n=20$, which covers the regime considered in most related works \cite{Dutta16,Tandon17}. However, beyond that we need to design numerically stable coding schemes and give up the optimal trade-off \eqref{eq:sim} between $d,s$ and $m$.  In this section we find an achievable region for which the condition numbers of all operations in the gradient reconstruction phase are upper bounded by a given value $\kappa$, so that the numerical stability can be guaranteed.
To that end, for any three given integers $n>n_1>n_2$,
we define a function $\gamma(n,n_1,n_2,\kappa)$ to be the smallest integer $n_3$ such that there is an $n_1 \times n$ matrix $V$ satisfying the following two properties:
\begin{enumerate}
\item $n_3\ge n_1$. For every subset $\cF\subseteq[n]$ with cardinality $|\cF|=n_3$, the condition number of $V_\cF V_\cF^T$ is no larger than $\kappa$, where $V_\cF$ is the submatrix of $V$ consisting of columns whose indices are in the set $\cF$.
\item Let $V_{[1:n_2]}$ be the submatrix of $V$ consisting of the first $n_2$ rows of $V$. We require that  every $n_2\times n_2$ submatrix consisting of circulant consecutive\footnote{``circulant consecutive" means that the indices $n$ and $1$ are considered consecutive. A more detailed explanation is given in Section~\ref{Sect:stbl}.} columns of $V_{[1:n_2]}$ is invertible.
\end{enumerate}
Note that  property 1) is similar to the restricted isometry property (RIP) property in compressed sensing \cite{Candes05}. The only difference is that in compressed sensing $n_3<n_1$ while here we require $n_3\ge n_1$.
We point out two obvious properties of this function:
(1) for a fixed triple $(n,n_1,n_2)$, the function $\gamma(n,n_1,n_2,\kappa)$ decreases with $\kappa$; (2) when $\kappa$ is large enough, $\gamma(n,n_1,n_2,\kappa)=n_1$.
We now state the theorem in the case $n=k$ to simplify the notation (see Remark~\ref{rm:n=k}):

\begin{theorem}\label{thm:stb}
Let $\kappa$ be the upper bound on the condition number of all the operations in the gradient reconstruction phase.
A triple $(d,s_\kappa,m)$ is achievable if 
\begin{equation}\label{eq:stb}
s_\kappa \le n -\gamma(n,n-d+m,n-d,\kappa).
\end{equation}
\end{theorem}
The proof of this theorem is given in Section~\ref{Sect:stbl}.
As discussed above, when $\kappa$ is large enough, i.e., when the stability constraint is loose, we have $\gamma(n,n-d+m,n-d,\kappa) = n-d+m$, and \eqref{eq:stb} becomes $s_\kappa \le d-m$, which is the same as \eqref{eq:sim}.
Moreover, since $\gamma(n,n-d+m,n-d,\kappa)$ decreases with $\kappa$, $s_\kappa$ increases with $\kappa$. Namely, we can tolerate more stragglers if we allow less numerical stability.
In our experiments we find that by setting $V$ to be Gaussian random matrix, we can achieve $s_\kappa = d-m$ with numerically stable coding scheme for $n\le 30$, which improves upon the coding scheme based on Vandermonde matrices.

By choosing $V$ as a Gaussian random matrix and
using the classical bounds on eigenvalues of large Wishart matrices\footnote{A Wishart matrix is a matrix of form $AA^T$, where $A$ is a Gaussian random matrix.} \cite{Geman80,Silverstein85} together with the union bound, we can obtain an upper bound of $\gamma(n,n_1,n_2,\kappa)$.
Let us introduce some more definitions to state the upper bound.
Given two integers $n>n_1$, define the function
$$
f_{n,n_1}(x) := \sqrt{\frac{n_1}{x}} + \sqrt{\frac{2n}{x} H(x/n)}
 \text{~for all~} n_1\le x \le n,
$$
where $H$ is the entropy function $H(q):=-q\ln q -(1-q)\ln(1-q)$ defined for $0<q<1$.
It is easy to verify that when $n_1/n>1/2$,  $f_{n,n_1}(x)$ strictly decreases with $x$.
Following the same steps\footnote{There are two differences between the settings in our paper and \cite{Candes05}: First, we have one more condition that certain $n_2\times n_2$ submatrices of $V_{[1:n_2]}$ must be invertible, but this is satisfied with probability $1$ for Gaussian random matrices, so this extra condition makes no difference to the proof, and the bound \eqref{eq:tfy} does not depend on $n_2$.
Second, in our paper we require $n_3\ge n_1$ while in \cite{Candes05} $n_3<n_1$, but this difference can also be resolved by a trivial modification of the proof in \cite{Candes05}.} as in the proof of 
\cite[Lemma 3.1]{Candes05}, we can show that when $n_1/n>1/2$ and $n$ is large,
\begin{equation}\label{eq:tfy}
\gamma(n,n_1,n_2,\kappa) \le f_{n,n_1}^{-1}(\frac{\sqrt{\kappa}-1}{\sqrt{\kappa}+1})
\text{~~~for~} \kappa > \Big( \frac{1+\sqrt{n_1/n}}{1-\sqrt{n_1/n}} \Big)^2 .
\end{equation}
\begin{corollary}
Let $\kappa$ be the upper bound on the condition number of all the operations in the gradient reconstruction phase.
When  $(d-m)/n<1/2$, $\kappa > \Big( \frac{1+\sqrt{n_1/n}}{1-\sqrt{n_1/n}} \Big)^2 $, and $n$ is large enough, a triple $(d,s_\kappa,m)$ is achievable if 
$$
s_\kappa \le n - f_{n,n_1}^{-1} \Big(\frac{\sqrt{\kappa}-1}{\sqrt{\kappa}+1} \Big).
$$
\end{corollary}

\section{Coding Scheme}\label{sect:code}
In this section, we present a coding scheme achieving \eqref{eq:sim} with equality, i.e., the parameters in our scheme satisfy $d=s+m$.
First we introduce two binary operations $\oplus$ and $\ominus$ over the set $[n]$.
For $a,b\in[n]$, define
$$
a\oplus b:=\left\{\begin{array}{cc} a+b & \text{~if~} a+b\le n \\
a+b-n & \text{~if~} a+b>n \end{array} \right.,
\quad
a\ominus b:=\left\{\begin{array}{cc} a-b & \text{~if~} a-b\ge 1 \\
a-b+n & \text{~if~} a-b\le 0 \end{array} \right. .
$$
In our scheme, each worker $W_i$ is assigned with $d$ data subsets $D_{i},D_{i\oplus 1},D_{i\oplus 2},\dots,D_{i\oplus (d-1)}$.
This is equivalent to say that each data subset $D_i$ is assigned to $d$ workers $W_{i},W_{i\ominus 1},W_{i\ominus 2},\dots,W_{i\ominus (d-1)}$.

\subsection{Proof of achievability part of Theorem~\ref{thm:main}}\label{sect:achieve}

Let $\theta_1,\theta_2,\dots,\theta_n$ be $n$ distinct real numbers.
Define $n$ polynomials $p_i,i\in[n]$,
\begin{equation}\label{eq:dpi}
p_i(x)=\prod_{j=1}^{n-d}(x-\theta_{i\oplus j}).
\end{equation}
Before proceeding further, let us explain the meaning of $\theta_i$ and $p_i$. Each $\theta_i$ is associated with the worker $W_i$, and each $p_i$ is associated with the dataset $D_i$. In our scheme, $p_j(\theta_i)\neq 0$ means that worker $W_i$ needs the value of $g_j$ to calculate $f_i(g_{i_1},g_{i_2},\dots,g_{i_d})$, and therefore $D_j$ is assigned to $W_i$. On the other hand, $p_j(\theta_i)= 0$ means that $D_j$ is not assigned to $W_i$.
By \eqref{eq:dpi}, we can see that each dataset $D_i$ is NOT assigned to $W_{i\oplus 1},W_{i\oplus 2},\dots,W_{i\oplus (n-d)}$.

Next we construct an $(mn)\times (n-s)$ matrix $B=(b_{ij})$ from the polynomials $p_i,i\in[n]$ defined in \eqref{eq:dpi}. 
Let $p_{i,j},j=0,1,\dots,n-s-1$ be the coefficients of the polynomial $p_i$, i.e.,
$$
p_i(x)=\sum_{j=0}^{n-s-1} p_{i,j}x^j.
$$
Since $\deg(p_i)=n-d$ and $d=s+m\ge s+1$, we have $p_{i,n-d}=1$ and $p_{i,n-d+1}=p_{i,n-d+2}=\dots=p_{i,n-s-1}=0$.
For every $i\in[n]$, we define $m$ polynomials $p_i^{(1)},p_i^{(2)},\dots,p_i^{(m)}$ recursively:
\begin{equation}\label{eq:rc}
\begin{aligned}
p_i^{(1)}(x) & := p_i(x),\\
p_i^{(u)}(x) & :=x p_i^{(u-1)}(x) - p_{i,n-d-1}^{(u-1)} p_i^{(1)}(x), \quad u=2,3,\dots,m,
\end{aligned}
\end{equation}
where $p_{i,j}^{(u)},j=0,1,\dots,n-s-1$ are the coefficients of $p_i^{(u)}$, i.e.,
$p_i^{(u)}(x)=\sum_{j=0}^{n-s-1} p_{i,j}^{(u)}x^j$.
Clearly, $p_i^{(u)}$ is a polynomial of degree $\deg(p_i^{(u)})=n-d+u-1$, and its leading coefficient is $1$, i.e.,
\begin{equation}\label{eq:b1}
\begin{aligned}
p_{i,n-d+u-1}^{(u)}=1 & \text{~for~} u=1,2,\dots,m,\\
p_{i,n-d+u}^{(u)}=p_{i,n-d+u+1}^{(u)}=\dots=p_{i,n-s-1}^{(u)}=0 & \text{~for~} u=1,2,\dots,m-1.
\end{aligned}
\end{equation}
It is also clear that $p_i | p_i^{(u)}$ for all $u\in[m]$ and all $i\in[n]$, so we have
$$
p_i^{(u)}(\theta_{i\oplus 1})=p_i^{(u)}(\theta_{i\oplus 2})=\dots=p_i^{(u)}(\theta_{i\oplus (n-d)})=0
\text{~for all~} u\in[m] \text{~and all~} i\in[n],
$$
which is equivalent to
\begin{equation}\label{eq:b2}
p_{i\ominus 1}^{(u)}(\theta_i)=p_{i\ominus 2}^{(u)}(\theta_i)=\dots=p_{i\ominus (n-d)}^{(u)}(\theta_i)=0 \text{~for all~} u\in[m] \text{~and all~} i\in[n].
\end{equation}
By a simple induction on $u$, one can further see that 
\begin{equation}\label{eq:b3}
p_{i,n-d}^{(u)}=p_{i,n-d+1}^{(u)}=\dots=p_{i,n-d+u-2}^{(u)}=0 \text{~for~} u=2,3,\dots,m.
\end{equation}
We can now specify the entries of $B$ as follows:
\begin{equation}\label{eq:db}
b_{(i-1)m+u,j}=p_{i,j-1}^{(u)} \text{~for all~}i\in[n],u\in[m],j\in\{1,2,\dots,n-s\}.
\end{equation}
By this definition, the following identity holds for every $x\in\mathbb{R}$:
{\footnotesize
\begin{equation}\label{eq:cw}
\begin{aligned}
& B [\begin{array}{ccccc} 1 & x & x^2 & \dots & x^{n-s-1} \end{array}]^T=\\
& [\begin{array}{*{13}c}
p_1^{(1)}(x) & p_1^{(2)}(x) & \dots & p_1^{(m)}(x) &
p_2^{(1)}(x) & p_2^{(2)}(x) & \dots & p_2^{(m)}(x) &
\dots\dots &
p_n^{(1)}(x) & p_n^{(2)}(x) & \dots & p_n^{(m)}(x) \end{array}]^T.
\end{aligned}
\end{equation}
}
Moreover, according to \eqref{eq:b1} and \eqref{eq:b3}, the submatrix $B_{[(n-d+1):(n-s)]}$ consisting of the last $m$ columns of $B$ is
\begin{equation}\label{eq:im}
B_{[(n-d+1):(n-s)]}=[\begin{array}{cccc} I_m & I_m & \dots & I_m \end{array}]^T,
\end{equation}
where $I_m$ is the $m\times m$ identity matrix, and there are $n$ identity matrix on the right-hand side of \eqref{eq:im}.

Recall that we assume $m|l$ throughout the paper. 
For every $v=0,1,\dots,l/m-1$ and $j\in[n]$, define an $m$-dimensional vector 
$$
y_v^{(j)}:=[\begin{array}{cccc}g_j(vm) & g_j(vm+1) & \dots & g_j(vm+m-1)\end{array}].
$$
For every $v=0,1,\dots,l/m-1$, define an $(mn)$-dimensional vector 
\begin{equation}\label{eq:dzv}
z_v := [\begin{array}{cccc}
y_v^{(1)} & y_v^{(2)} & \dots & y_v^{(n)} \end{array}].
\end{equation}
According to \eqref{eq:cw}, 
\begin{equation}\label{eq:ind}
z_v B [\begin{array}{ccccc} 1 & \theta_i & \theta_i^2 & \dots & \theta_i^{n-s-1} \end{array}]^T
=\sum_{j=1}^n \sum_{u=1}^m p_j^{(u)}(\theta_i) g_j(vm+u-1)
=\sum_{j=0}^{d-1} \sum_{u=1}^m p_{i\oplus j}^{(u)}(\theta_i) g_{i\oplus j}(vm+u-1),
\end{equation}
where the second equality follows from \eqref{eq:b2}.

Now we are ready to define the transmitted vector $f_i(g_i,g_{i\oplus 1},\dots,g_{i\oplus (d-1)})$ for each worker $W_i,i\in[n]$:
\begin{equation}\label{eq:df}
f_i(g_i,g_{i\oplus 1},\dots,g_{i\oplus (d-1)}):=  \left[
\begin{array}{c} z_0 \\ z_1 \\ \vdots \\ z_{l/m-1} \end{array}\right]
B [\begin{array}{ccccc} 1 & \theta_i & \theta_i^2 & \dots & \theta_i^{n-s-1} \end{array}]^T .
\end{equation}
By \eqref{eq:ind}, the value of $f_i(g_i,g_{i\oplus 1},\dots,g_{i\oplus (d-1)})$ indeed only depends on the values of $g_i,g_{i\oplus 1},\dots,g_{i\oplus (d-1)}$.

To complete the description of our coding scheme, we only need to show that for any subset $\cF\subseteq[n]$ with cardinality $|\cF|=n-s$, we can calculate 
$g_1+g_2+\dots+g_n$ from $\{f_i(g_i,g_{i\oplus 1},\dots,g_{i\oplus (d-1)}):i\in\cF\}$.
Let the column vectors $\{e_1,e_2,\dots,e_{n-s}\}$ be the standard basis of $\mathbb{R}^{n-s}$, i.e., all coordinates of $e_i$ are $0$ except the $i$th coordinate which is $1$.
By \eqref{eq:im}, we have
\begin{equation}\label{eq:ze}
z_v B e_{n-d+u} = \sum_{j=1}^n g_j(vm+u-1) \text{~for all~} 0\le v \le l/m-1 \text{~and all~} u\in[m].
\end{equation}
Without loss of generality let us assume that $\cF=\{1,2,\dots,n-s\}$.
Define the following $(n-s)\times (n-s)$ matrix
\begin{equation}\label{eq:vd}
A:=
\left[\begin{array}{cccc} 1 & 1 & \dots & 1 \\
\theta_1 & \theta_2 & \dots & \theta_{n-s} \\
\theta_1^2 & \theta_2^2 & \dots & \theta_{n-s}^2 \\
\vdots & \vdots & \vdots & \vdots \\
\theta_1^{n-s-1} & \theta_2^{n-s-1} & \dots & \theta_{n-s}^{n-s-1} \end{array} \right].
\end{equation}
 According to \eqref{eq:df}, from $\{f_i(g_i,g_{i\oplus 1},\dots,g_{i\oplus (d-1)}):i\in\cF\}$ we can obtain  the values of
\begin{equation}\label{eq:vt}
z_v B A
\text{~for all~} 0\le v \le l/m-1.
\end{equation}
Since $A$
is invertible, we can calculate $z_v B e_{n-d+u}$ for all $0\le v \le l/m-1$ and all $u\in[m]$ from the vectors in \eqref{eq:vt} by multiplying $A^{-1} e_{n-d+u}$ to the right.
By \eqref{eq:ze}, 
\begin{align*}
\{z_v B e_{n-d+u}:v\in\{0,1,\dots,l/m-1\},u\in[m]\}
& = \{\sum_{j=1}^n g_j(vm+u-1) : v\in\{0,1,\dots,l/m-1\},u\in[m]\} \\
& = \{\sum_{j=1}^n g_j(i):i\in\{0,1,\dots,l-1\}\}.
\end{align*}
Therefore we conclude that the sum vector $g_1+g_2+\dots+g_n$ can be calculated from $\{f_i(g_i,g_{i\oplus 1},\dots,g_{i\oplus (d-1)}):i\in\cF\}$ whenever $|\cF|=n-s$.
Thus we have shown that our coding scheme satisfies all three conditions in Definition~\ref{def:m}, and this completes the proof of the achievability part of Theorem~\ref{thm:main}.

\subsection{Efficient implementation of our coding scheme}
To implement our coding scheme,
an important step is to calculate the product 
$$
B [\begin{array}{ccccc} 1 & \theta_i & \theta_i^2 & \dots & \theta_i^{n-s-1} \end{array}]^T
$$
in order to obtain the transmitted vectors in \eqref{eq:df}. According to \eqref{eq:cw}, this product can be easily calculated by the recursive procedure \eqref{eq:rc}.
Notice that in this recursive procedure we need to know the values of $p_{i,n-d-1}^{(u-1)}$ for $u=2,3,\dots,m$, and by \eqref{eq:db} we have
$b_{(i-1)m+u,n-d}=p_{i,n-d-1}^{(u)}$. Therefore in our implementation we need to calculate (at least some of) the entries of the matrix $B$.
The entries of $B$ are specified in \eqref{eq:db}, and they are calculated recursively according to \eqref{eq:rc} from the coefficients of the polynomials $p_i,i\in[n]$.
While the recursive procedure in \eqref{eq:rc} might seem complicated, Algorithm~\ref{alg:B} below describes an efficient way to calculate $B$ from the coefficients of $p_i,i\in[n]$ defined in \eqref{eq:dpi}.

Finally, we remark that the examples in Fig.~\ref{sf:1} and Fig.~\ref{sf:2} are both obtained by setting $\theta_1=-2,\theta_2=-1,\theta_3=0,\theta_4=1,\theta_5=2$ in our coding scheme.

\begin{algorithm}
\caption{Algorithm to calculate the entries of $B$}\label{alg:B}
\textbf{Input:} $p_{i,j},i\in[n],j=0,1,\dots,n-d$, the coefficients of $p_i,i\in[n]$, i.e.,
$p_i(x)=\sum_{j=0}^{n-d} p_{i,j}x^j$.

\textbf{Output:} The $(mn)\times (n-s)$ matrix $B=(b_{ij})_{1\le i\le mn,1\le j\le n-s}$

\vspace*{0.05in}
\begin{algorithmic}
\State Initialize $B$ as a zero matrix
\For {$i=1,2,\dots,n$} 
\For {$j=1,2,\dots,n-d+1$}
\State $b_{(i-1)m+1,j} \gets p_{i,j-1}$
\EndFor
\EndFor
\For {$u=2,3,\dots,m$}
\For {$i=1,2,\dots,n$} 
\For {$j=2,3,\dots,n-d+u$}
\State $b_{(i-1)m+u,j} \gets b_{(i-1)m+u-1,j-1}$
\EndFor
\For {$j=1,2,\dots,n-d+1$}
\State $b_{(i-1)m+u,j} \gets b_{(i-1)m+u,j}-b_{(i-1)m+u,n-d+1} b_{(i-1)m+1,j}$
\EndFor
\EndFor
\EndFor
\State \textbf{return} $B$
\end{algorithmic}
\end{algorithm}

\subsection{Choice of $\{\theta_1,\theta_2,\dots,\theta_n\}$ and numerical stability}
In Section~\ref{sect:achieve} we have shown that our coding scheme works for any set of $n$ distinct real numbers $\{\theta_1,\theta_2,\dots,\theta_n\}$. However, in the proof we assume that the computation has infinite precision, which is not possible in real world application.
Stability aspects need to be considered for the inversion of Vandermonde matrices of form \eqref{eq:vd} when the master node reconstructs the full gradient vector from partial gradient vectors returned by the first $(n-s)$ worker nodes. It is well known that the accuracy of matrix inversion depends heavily on the condition number of the matrix. Therefore we need to find a set of $\{\theta_1,\theta_2,\dots,\theta_n\}$ such that every $(n-s)\times (n-s)$ submatrix of the following matrix $V$ has low condition number.
\begin{equation}\label{eq:defv}
V:=
\left[\begin{array}{ccccc} 1 & 1 & 1 & \dots & 1 \\
\theta_1 & \theta_2 & \theta_3 & \dots & \theta_n \\
\theta_1^2 & \theta_2^2 & \theta_3^2 & \dots & \theta_n^2 \\
\vdots & \vdots & \vdots & \vdots & \vdots \\
\theta_1^{n-s-1} & \theta_2^{n-s-1} & \theta_3^{n-s-1} & \dots & \theta_n^{n-s-1} \end{array} \right].
\end{equation}
In our implementation in Section~\ref{exp1}, we choose 
\begin{equation}\label{eq:tt}
\{\theta_1,\theta_2,\dots,\theta_n\} =
\left\{\begin{array}{cc} \{\pm(1+i/2), i=0,1,2,\dots,n/2-1 \} & \mbox{for even } n \\
\{0,\pm(1+i/2), i=0,1,2,\dots,(n-1)/2-1 \} & \mbox{for odd } n \end{array}\right.   .
\end{equation}
We test this choice for various values of $n$, and we find that when $n\le 20$, our scheme is numerically stable for all possible values of $d,s$ and $m$. More specifically, the relative error (measured in $\ell_\infty$ norm) between reconstructed full gradient vector at the master node and the true value is less than $0.2\%$.
However, the numerical stability deteriorates very quickly as $n$ becomes larger than $20$: when $n=23$, the relative error in the worst case can be up to $80\%$, and when $n=26$, our algorithm crushes.

Note that numerical instability of our coding scheme is NOT due to the introduction of the communication cost reduction factor $m$. In fact, in \cite{Halbawi17,Raviv17} the authors presented coding schemes to achieve \eqref{eq:sim} for the special case of $m=1$, and the schemes in both paper also involve inversion of Vandermonde matrices, so they also suffer from numerical instability.
Moreover, the schemes in both paper set $\theta_1,\theta_2,\dots,\theta_n$ to be roots of unity. Such a choice does not resolve the numerical instability issue either: it is shown in \cite{Pan16} that in the worst case the condition number of $(n-s)\times(n-s)$ submatrices of $V$ grows exponentially fast in $n$ when $\theta_1,\theta_2,\dots,\theta_n$ are roots of unity.

\section{Proof of Theorem~\ref{thm:stb}} \label{Sect:stbl}
Let $s:=d-m$ and
let $s_\kappa = n-\gamma(n,n-s,n-d,\kappa).$ Then by definition of the function $\gamma$, there is an $(n-s) \times n$ matrix $V$ such that
\begin{enumerate}
\item $n- s_\kappa \ge n-s$.
For every subset $\cF\subseteq[n]$ with cardinality $|\cF|=n-s_\kappa$, the condition number of $V_\cF V_\cF^T$ is no larger than $\kappa$, where $V_\cF$ is the submatrix of $V$ consisting of columns whose indices are in the set $\cF$.
\item Define $n$ submatrices of $V$ with size $(n-d)\times (n-d)$ as follows: For $i=1,2,\dots,d+1$,
define $S_i$ to be the submatrix of $V$ corresponding to the  row indices $\{1,2,\dots,n-d\}$ and the column indices $\{i,i+1,\dots,i+n-d-1\}$. For $i=d+2,d+3,\dots,n$, define $S_i$ to be the submatrix of $V$ corresponding to the row indices $\{1,2,\dots,n-d\}$ and the column indices $\{i,i+1,\dots,n,1,2,\dots,i-d-1\}$. The matrix $S_i$ is invertible for all $i\in[n]$.
\end{enumerate}

We further define another $n$ submatrices of $V$ with size $m\times (n-d)$ as follows: For $i=1,2,\dots,d+1$,
define $R_i$ to be the submatrix of $V$ corresponding to the  row indices $\{n-d+1,n-d+2,\dots,n-s\}$ and the column indices $\{i,i+1,\dots,i+n-d-1\}$. For $i=d+2,d+3,\dots,n$, define $R_i$ to be the submatrix of $V$ corresponding to the row indices $\{n-d+1,n-d+2,\dots,n-s\}$ and the column indices $\{i,i+1,\dots,n,1,2,\dots,i-d-1\}$.

To prove Theorem~\ref{thm:stb}, we only need to find an $(mn)\times (n-s)$ matrix $B$ satisfying the following two conditions:
\begin{enumerate}
\item The product of the $i$th row of $B$ and the $j$th column of $V$ is $0$ for all $j\in[n]$ and all $i\in[mn]\setminus\{(mj)\text{mod} (mn)+1,(mj+1)\text{mod} (mn)+1,(mj+2)\text{mod} (mn)+1,\dots,(mj+dm-1)\text{mod} (mn)+1\}$; 
\item Equation \eqref{eq:im}.
\end{enumerate}
Since Equation \eqref{eq:im} already specifies the last $m$ columns of matrix $B$, we only need to design the first $(n-d)$ columns.
For $i\in[n]$, we write $B_i$ the $m\times(n-d)$ submatrix of $B$ corresponding to row indices $\{(i-1)m+1,(i-1)m+2,\dots,im\}$ and column indices $\{1,2,\dots,n-d\}$. 
Now the condition 1) above is equivalent to 
\begin{equation}\label{eq:nuh}
[B_i \quad I_m] [S_i^T \quad R_i^T]^T = 0 \text{~for all~} i\in[n].
\end{equation}
Since $S_i$ is invertible for all $i\in[n]$, the equation above is equivalent to
$$
[B_i \quad I_m] [I_{n-d} \quad (R_i S_i^{-1})^T]^T = 0 \text{~for all~} i\in[n].
$$
It is easy to see that we can set $B_i:=-R_i S_i^{-1}$ for all $i\in[n]$ to satisfy this constraint. 
As a result, the matrix $B$ in our coding scheme is
$$
B:= \left[\begin{array}{ccccc}
(-R_1 S_1^{-1})^T & (-R_2 S_2^{-1})^T  & (-R_3 S_3^{-1})^T  & \dots & (-R_n S_n^{-1})^T \\
I_m & I_m & I_m & \dots & I_m
 \end{array} \right]^T,
$$
where the matrices $R_i,S_i,i\in[n]$ are defined above as certain submatrices of the matrix $V$.

 Denote $V_i$ as the $i$th column of $V$. Recall the definition of $z_v,v=0,1,\dots,l/m-1$ in \eqref{eq:dzv}.
Now we are ready to define the transmitted vector $f_i(g_i,g_{i\oplus 1},\dots,g_{i\oplus (d-1)})$ for each worker $W_i,i\in[n]$:
\begin{equation}\label{eq:lpl}
f_i(g_i,g_{i\oplus 1},\dots,g_{i\oplus (d-1)}):=  \left[
\begin{array}{c} z_0 \\ z_1 \\ \vdots \\ z_{l/m-1} \end{array}\right]
B V_i .
\end{equation}
By \eqref{eq:nuh}, the value of $f_i(g_i,g_{i\oplus 1},\dots,g_{i\oplus (d-1)})$ indeed only depends on the values of $g_i,g_{i\oplus 1},\dots,g_{i\oplus (d-1)}$.

To complete the description of our coding scheme, we only need to show that for any subset $\cF\subseteq[n]$ with cardinality $|\cF|=n-s_\kappa$, we can calculate 
$g_1+g_2+\dots+g_n$ from $\{f_i(g_i,g_{i\oplus 1},\dots,g_{i\oplus (d-1)}):i\in\cF\}$, and the condition numbers of all operations in the gradient reconstruction phase are upper bounded by $\kappa$.
Let the column vectors $\{e_1,e_2,\dots,e_{n-s}\}$ be the standard basis of $\mathbb{R}^{n-s}$.
 According to \eqref{eq:lpl}, from $\{f_i(g_i,g_{i\oplus 1},\dots,g_{i\oplus (d-1)}):i\in\cF\}$ we can obtain  the values of
\begin{equation}\label{eq:cyu}
z_v B V_\cF
\text{~for all~} 0\le v \le l/m-1.
\end{equation}
Similarly to the coding scheme in Section~\ref{sect:achieve}, we can calculate $z_v B e_{n-d+u}$ for all $0\le v \le l/m-1$ and all $u\in[m]$ from the vectors in \eqref{eq:cyu} by multiplying $V_\cF^T (V_\cF V_\cF^T)^{-1} e_{n-d+u}$ to the right.
Therefore we conclude that the sum vector $g_1+g_2+\dots+g_n$ can be calculated from $\{f_i(g_i,g_{i\oplus 1},\dots,g_{i\oplus (d-1)}):i\in\cF\}$ whenever $|\cF|=n-s_\kappa$.
Thus we have shown that our coding scheme satisfies all three conditions in Definition~\ref{def:m}. Moreover, the only matrix inversions in the gradient reconstruction phase is calculating $(V_\cF V_\cF^T)^{-1}$. By definition of the matrix $V$, the condition numbers of  $\{V_\cF V_\cF^T:\cF\subseteq[n],|\cF|=n-s_\kappa\}$ are all upper bounded by $\kappa$. This completes the proof of Theorem~\ref{thm:stb}.

As a final remark, we do not impose any stability constraints on the calculation of $S_i^{-1}$ when constructing the matrix $B$, which is also a matrix inversion. This is because the construction of $B$ is only one-time, so we can afford to use high-precision calculation to compensate for possibly large condition number in the construction of matrix $B$.

\subsection{Choice of the matrix $V$}
In Section~\ref{sect:code} we set $V$ to be a (non-square) Vandermonde matrix.
However, it is well known that Vandermonde matrices are badly ill-conditioned \cite{Gautschi87},
so one way to alleviate the numerical instability is to use random matrices instead of Vandermonde matrices. 
For instance, we can choose $V$ in \eqref{eq:defv} to be a Gaussian random matrix and design the matrix $B$ as described above. 
According to our experimental results, using random matrices allows our scheme to be numerically stable for all $n\le 30$ and all possible values of $d,s$ and $m$.


\section{Experiments on Amazon EC2 clusters}\label{exp1}

In this section, we use our proposed gradient coding scheme to train a logistic regression model on the Amazon Employee Access dataset from Kaggle\footnote{https://www.kaggle.com/c/amazon-employee-access-challenge}, and we compare the running time and Generalization AUC\footnote{AUC is short for area under the ROC-curve. The Generalization AUC can be efficiently calculated using the ``sklearn.metrics.auc" function in Python.} between our method and baseline approaches. More specifically, we compare our scheme against: (1) the {\em naive} scheme, where the data is uniformly divided among all workers without replication and the master node waits for all workers to send their results before updating model parameters in each iteration, and (2) the coding schemes in \cite{Tandon17,Halbawi17,Raviv17}, i.e., the special case of $m=1$ in our scheme.
Note that in \cite{Tandon17} the authors implemented their methods (which is the special case of $m=1$ in this paper) to train the same model over the same dataset.

We used Python with {\fontfamily{qcr}\selectfont mpi4py} package to implement our gradient coding schemes proposed in Section~\ref{sect:achieve}, where $\theta_1,\theta_2,\dots,\theta_n$ are specified in \eqref{eq:tt}.
We used {\fontfamily{qcr}\selectfont t2.micro} instances on Amazon EC2 as worker nodes and a single {\fontfamily{qcr}\selectfont c3.8xlarge} instance as the master node.

As a common preprocessing step, we converted the categorical features in the Amazon Employee Access dataset to binary features by one-hot encoding, which can be easily realized in Python.
After one-hot encoding with interaction terms, the dimension of parameters in our model is $l=343474$.
For all three schemes (our proposed scheme, the schemes in \cite{Tandon17,Halbawi17,Raviv17} and the naive scheme), we used $N=26220$ training samples and adopted Nesterov's Accelerated Gradient (NAG) descent \cite[Section 3.7]{Bubeck15} to train the model.
These experiments were run on $n=10,15,20$ worker nodes.

In Fig.~\ref{fig:tpi}, we compare average running time per iteration for different schemes. For coding schemes proposed in \cite{Tandon17,Halbawi17,Raviv17}, i.e., coding schemes corresponding to $m=1$ in our paper, we choose the optimal value of $s$ such that it has the smallest running time among all possible choices of $(m=1,s)$. For coding schemes proposed in this paper, i.e., schemes with $m>1$, we choose two pairs of $(m,s)$ with the smallest running time among all possible choices.
We can see that for all three choices of $n$, our scheme outperforms the schemes in \cite{Tandon17,Halbawi17,Raviv17} by at least $23\%$ and outperforms the naive scheme by at least $32\%$.
We then plot generalization AUC vs. running time for these choices of $(m,s)$ in Fig.~\ref{fig:avt}. The curves corresponding to $m>1$ are always on the left side of the curves corresponding to $m=1$ and the naive scheme, which means that our schemes achieve the target generalization error much faster than the other two schemes.

\begin{figure}[h] 
\includegraphics[width=\textwidth]{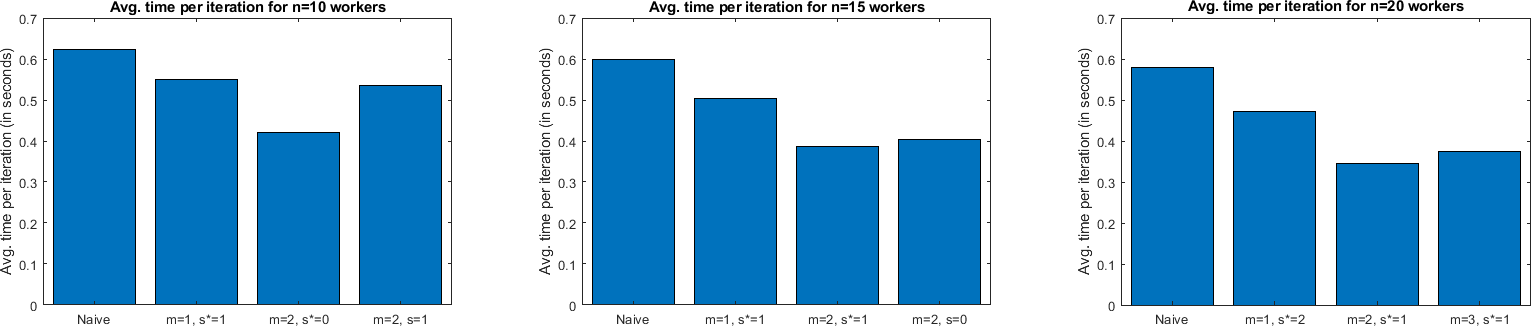}
\caption{Avg. time per iteration for $n=10,15,20$ workers, $s^*$ means that it is the optimal value of $s$ for that choice of $m$} \label{fig:tpi}
\end{figure}

\begin{figure}[h]
\includegraphics[width=\textwidth]{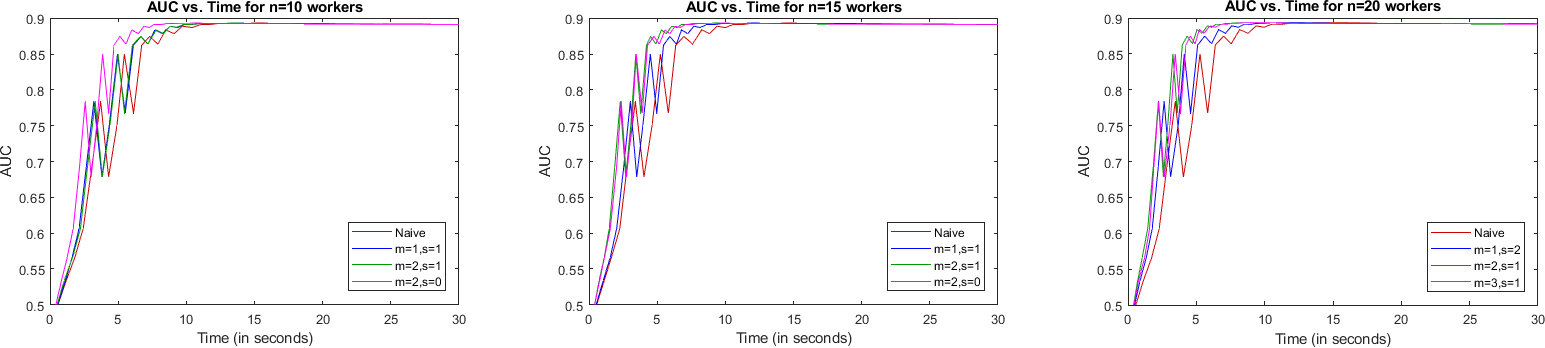}
\caption{AUC vs. Time for $n=10,15,20$ workers. The curves corresponding to $m>1$ are always on the left side of the curves corresponding to $m=1$ and the naive scheme, which means that our schemes achieve the target generalization error much faster than the other schemes.}
 \label{fig:avt}
\end{figure}

\section{Analysis of the total computation and communication time}\label{ana1}
In this section we analyze the total runtime of our coding scheme for different choices of the design parameters $(d,s,m)$ based on a probabilistic model. Our analysis reveals the optimal choice of parameters under some special cases and also sheds light on how to choose $(d,s,m)$ in general.
Following the probabilistic model of runtime in \cite{Lee16}, we assume that both computation time and communication time have shifted exponential distribution, which is the sum of a constant and an exponential random variable.
We also assume that for each worker, the computation time is proportional to $d$, the number of assigned data subsets,
and the communication time is proportional to the dimension of transmitted vector. This assumption is in accordance to the observation in the experiments of \cite{Tandon17}.
The total runtime is the sum of the computation time and the communication time.\footnote{Since the total number of samples $N$ in large-scale machine learning tasks is of order hundreds of millions, we have $N \gg n$ in our problem. The computation time is of order $\Theta(Nl)$ while the reconstruction time is of order $O(n l)$. Therefore we can ignore the reconstruction phase at the master node when estimating the total runtime.}

Formally speaking, for $i,j\in[n]$,
let $T_{i,j}^{(1)}$ be the computation time of data subset $D_j$ for worker $W_i$.
Similarly, for $i\in[n]$, let $T_{i}^{(2)}$ be the communication time for worker $W_i$ to send a vector of dimension $l$.
We make the following assumption:
\begin{enumerate}
\item For $i\in[n]$, $T_{i,1}^{(1)}=T_{i,2}^{(1)}=\dots=T_{i,n}^{(1)}=T_i^{(1)}$, where
the random variables $T_i^{(1)},i\in[n]$ are i.i.d. with distribution
$$
\Pr(T_i^{(1)}\le t)=1-e^{-\lambda_1(t-t_1)}, \forall t\ge t_1.
$$
\item The communication time for worker $W_i$ to send a vector of dimension $l'$ is $(l'/l)T_i^{(2)}$,
where the random variables $T_i^{(2)},i\in[n]$ are i.i.d. with distribution
$$
\Pr(T_i^{(2)}\le t)=1-e^{-\lambda_2(t-t_2)}, \forall t\ge t_2.
$$
\item The random variables $T_i^{(1)},i\in[n]$ and $T_i^{(2)},i\in[n]$ are mutually independent.
\end{enumerate}
Here $t_1$ and $t_2$ are the minimum computation and communication time of a worker in perfect conditions, respectively; $\lambda_1$ and $\lambda_2$ depict the straggling behavior in the computation and communication process, respectively.
It is clear that smaller $\lambda_1$ means the distribution of the computation time has a heavy tail and more likely to cause delay. Similarly, smaller $\lambda_2$ means that the communication process is more likely to be the bottleneck.

Under the assumptions above, for a triple $(d,s,m)$, the computation time of worker $W_i$ is $dT_i^{(1)}$, which is the sum of the constant $dt_1$ and an exponential random variable with distribution $\Exp(\lambda_1/d)$, and the communication time of worker $W_i$ is $\frac{1}{m}T_i^{(2)}$,
which is the sum of the constant $t_2/m$ and an exponential random variable with distribution
$\Exp(m\lambda_2)$.
Therefore, the total runtime for each worker $W_i$ is the sum of $dt_1+t_2/m$ and a random variable 
$T_i^{(d,m)}$ with distribution\footnote{\eqref{eq:dm} gives the expression when $\lambda_1/d \neq m\lambda_2$. When $\lambda_1/d = m\lambda_2$, $T_i^{(d,m)}$ is an Erlang random variable with parameters $2$ and $m\lambda_2$.}
\begin{equation}\label{eq:dm}
\Pr(T_i^{(d,m)}\le t)=1-\frac{\lambda_1/d}{\lambda_1/d-m\lambda_2}e^{-m\lambda_2 t}
-\frac{m\lambda_2}{m\lambda_2-\lambda_1/d}e^{-(\lambda_1/d) t},\quad \forall t\ge 0.
\end{equation}
Since $T_i^{(d,m)},i\in[n]$ are i.i.d. and we only need to wait for the first $n-s$ workers to return their results, the total runtime of the whole task is
\begin{equation}\label{eq:tot}
T_{\tot}=dt_1+t_2/m+T_{d,s,m},
\end{equation}
 where the random variable $T_{d,s,m}$ has distribution
\begin{equation}\label{eq:hard}
\begin{aligned}
\Pr(T_{d,s,m} \le t)= & \int_0^t \frac{n!}{(n-s-1)!s!} \frac{m\lambda_1\lambda_2}{\lambda_1-dm\lambda_2}
\Big( 1-\frac{\lambda_1/d}{\lambda_1/d-m\lambda_2}e^{-m\lambda_2 \tau}
-\frac{m\lambda_2}{m\lambda_2-\lambda_1/d}e^{-(\lambda_1/d) \tau} \Big)^{n-s-1} \\
& \Big(\frac{\lambda_1/d}{\lambda_1/d-m\lambda_2}e^{-m\lambda_2 \tau}
+ \frac{m\lambda_2}{m\lambda_2-\lambda_1/d}e^{-(\lambda_1/d) \tau} \Big)^s
\Big( e^{-m\lambda_2 \tau}  - e^{-(\lambda_1/d) \tau} \Big) d \tau, \quad \forall t\ge 0.
\end{aligned}
\end{equation}
Here $T_{d,s,m}$ is the $(n-s)$th order statistics of the distribution \eqref{eq:dm}.
Since \eqref{eq:sim} depicts the optimal tradeoff between $d,s$ and $m$, we should choose these three parameters to achieve \eqref{eq:sim} with equality in order to minimize $T_{\tot}$.
In other words, we should always set $s=d-m$.

To understand how the choice of $(d,s,m)$ affects the total runtime, let us first consider two extreme cases.

\underline{\it Computation time is dominant}: Assume that $\lambda_1\ll \lambda_2$ and $t_1\gg t_2$, so that we can ignore the communication time. Obviously we should set $m=1$ and therefore $s=d-1$.
In this case, $T_{d,d-1,1}$ is the $(n-d+1)$th order statistics of $n$ i.i.d exponential random variables with distribution $\Exp(\lambda_1/d)$.
Consequently, 
$$
\mathbb{E}[T_{d,d-1,1}]=\frac{d}{\lambda_1}(\sum_{i=0}^{n-d}\frac{1}{n-i}),
$$
and the total expected runtime is 
\begin{equation}\label{eq:sk}
\mathbb{E}[T_{\tot}]=dt_1+\frac{d}{\lambda_1}(\sum_{i=0}^{n-d}\frac{1}{n-i}).
\end{equation}

\begin{proposition}\label{prop:p1}
When $\lambda_1 t_1<\frac{1}{n-1}(\sum_{i=2}^n 1/i)$, we should choose $d=n$ to minimize $\mathbb{E}[T_{\tot}]$, i.e., each worker is assigned all datasets $D_1,\dots,D_n$.
When $\lambda_1 t_1\ge\frac{1}{n-1}(\sum_{i=2}^n 1/i)$, we should choose $d=1$ to minimize $\mathbb{E}[T_{\tot}]$, i.e., each worker is assigned only one dataset.
\end{proposition}
The proof of this proposition is given in Appendix~\ref{ap:p1}.

\underline{\it Communication time is dominant}: Assume that $\lambda_1\gg \lambda_2$ and $t_1\ll t_2$, so that we can ignore the computation time. Obviously we should set $d=n$ and therefore $s=n-m$.
In this case, $T_{n,n-m,m}$ is the $m$th order statistics of $n$ i.i.d exponential random variables with distribution $\Exp(m\lambda_2)$.
Consequently, 
$$
\mathbb{E}[T_{n,n-m,m}]=\frac{1}{m\lambda_2}(\sum_{i=0}^{m-1}\frac{1}{n-i}),
$$
and the total expected runtime is 
$$
\mathbb{E}[T_{\tot}]=\frac{t_2}{m}+\frac{1}{m\lambda_2}(\sum_{i=0}^{m-1}\frac{1}{n-i}).
$$

For a fixed value of $n$, if $t_2\gg \frac{1}{\lambda_2}$, then the optimal choice is $m=n$. On the other hand, if $t_2\ll \frac{1}{\lambda_2}$, then the optimal choice is $m=1$.

Now let us fix the values of $t_2$ and $\lambda_2$, and let $n$ grow. We want to find the optimal rate $\alpha:=m/n$ to minimize $\mathbb{E}[T_{\tot}]$. In this regime, we use the approximation
\begin{equation}\label{eq:qnm}
\mathbb{E}[T_{\tot}]\approx \frac{t_2}{m}+\frac{1}{m\lambda_2}\log\frac{n}{n-m}
=\frac{1}{\alpha n}(t_2-\frac{1}{\lambda_2}\log(1-\alpha)),
\end{equation}
\begin{proposition}\label{prop:p2}
The optimal ratio $\alpha$ between the communication cost reduction factor $m$ and the number of workers $n$ is the unique root of the following equation
$$
\frac{\alpha}{1-\alpha}+\log(1-\alpha)=\lambda_2 t_2.
$$
\end{proposition}
Note that for any given positive $\lambda_2$ and $t_2$, the equation above has a unique root in the open interval $(0,1)$.
The proof of this proposition is given in Appendix~\ref{ap:p2}

\subsection{Numerical analysis}
When computation time and communication time are comparable, accurate analysis of \eqref{eq:hard} becomes more difficult. Here we use a numerical example to illustrate the advantages of our new proposal.
According to \eqref{eq:tot} and \eqref{eq:hard}, when $\lambda_1/d \neq m\lambda_2$,
\begin{align*}
\mathbb{E}[T_{\tot}]= & dt_1+t_2/m \\
 + & \int_0^\infty \frac{n!}{(n-s-1)!s!} \frac{m\lambda_1\lambda_2}{\lambda_1-dm\lambda_2}
\Big( 1-\frac{\lambda_1/d}{\lambda_1/d-m\lambda_2}e^{-m\lambda_2 t}
-\frac{m\lambda_2}{m\lambda_2-\lambda_1/d}e^{-(\lambda_1/d) t} \Big)^{n-s-1} \\
& \Big(\frac{\lambda_1/d}{\lambda_1/d-m\lambda_2}e^{-m\lambda_2 t}
+ \frac{m\lambda_2}{m\lambda_2-\lambda_1/d}e^{-(\lambda_1/d) t} \Big)^s
\Big( e^{-m\lambda_2 t}  - e^{-(\lambda_1/d) t} \Big) t d t.
\end{align*}
When $\lambda_1/d = m\lambda_2$,
\begin{align*}
 & \mathbb{E}[T_{\tot}]=  dt_1+t_2/m \\
 & +  \int_0^\infty \frac{n! m^2 \lambda_2^2}{(n-s-1)!s!} 
\Big( 1-e^{-m\lambda_2 t}
-m\lambda_2 t e^{-m\lambda_2 t} \Big)^{n-s-1} 
 \Big( e^{-m\lambda_2 t} + m\lambda_2 t e^{-m\lambda_2 t} \Big)^s
 e^{-m\lambda_2 t}  t^2 d t.
\end{align*}

In the following table we take $n=k=8,\lambda_1=0.8,\lambda_2=0.1,t_1=1.6,t_2=6$, and we list $\mathbb{E}[T_{\tot}]$ for all possible choices of $d$ and $m$. Recall that we take $s=d-m$ to minimize $T_{\tot}$.
\begin{center}
\begin{tabular}{| c | c | c | c | c | c | c | c | c |}
\hline
 \backslashbox{$m$}{$d$} & 1 & 2 & 3 & 4 & 5 & 6 & 7 & 8 \\ \hline
1 & 36.1138 & 29.2288 & 27.3351 & 26.7469 & 26.4574 & 26.0891 & 25.4172 & 24.1063  \\ \hline
2 &  & 23.1036 & 21.3994 & 21.5369 & 21.9114 & 22.2099 & 22.3189 & 22.1405 \\ \hline
3 &  &        & 22.2604 & 21.3697 & 21.5749 & 21.9095 & 22.1707 & 22.2772 \\ \hline
4 &  &        &   & 24.8036 & 23.2793 & 23.1114 & 23.1862 & 23.2611  \\ \hline
5 &  &   &   &   & 28.5800 & 25.9827 & 25.2862  & 25.0141  \\ \hline
6 &  &   &   &   &   & 32.8664 & 29.0745  & 27.7904 \\ \hline
7 &  &   &   &   &   &   & 37.3977  & 32.3759 \\ \hline
8 &  &   &   &   &   &   &    & 42.0638 \\ \hline
\end{tabular}
\end{center}
We can see that $d=4,m=3$ is the optimal choice, whose total runtime is $21.3697$. The runtime for uncoded scheme ($d=m=1$) is $36.1138$, and the best achievable runtime for the coding schemes in 
\cite{Tandon17,Halbawi17,Raviv17} is $24.1063$ (d=8,m=1).
Therefore our coding scheme outperforms the uncoded scheme by $41\%$ and outperforms the schemes in
\cite{Tandon17,Halbawi17,Raviv17} by $11\%$.

Next we investigate how the optimal triple $(d,s,m)$ varies with the values of $\lambda_1,\lambda_2,t_1,t_2$. First we fix $\lambda_1,t_1$, and let $\lambda_2,t_2$ vary.
In the following table we take $n=k=10,\lambda_1=0.6$ and $t_1=1.5$. The optimal triple $(d,s,m)$ for different values of $\lambda_2$ and $t_2$ is recorded in the table.
\begin{center}
\begin{tabular}{| c | c | c | c | c | c | c | c |}
\hline
 \backslashbox{$\lambda_2$}{$t_2$} & 1.5 & 3 & 6 & 12 & 24 & 48 & 96 \\ \hline
0.05 & (10,9,1) & (10,8,2) & (10,8,2) & (10,7,3) & (10,6,4) & (10,5,5) & (10,4,6)  \\ \hline
0.1 & (3,1,2) & (3,1,2) & (3,1,2) & (4,1,3) & (4,1,3) & (10,5,5) & (10,4,6)  \\ \hline
0.15 & (2,0,2) & (2,0,2) & (2,0,2) & (2,0,2) & (4,1,3) & (10,6,4) & (10,4,6)  \\ \hline
0.2 & (2,0,2) & (2,0,2) & (2,0,2) & (2,0,2) & (2,0,2) & (10,6,4) & (10,4,6)  \\ \hline
0.25 & (2,0,2) & (2,0,2) & (2,0,2) & (2,0,2) & (2,0,2) & (10,6,4) & (10,4,6)  \\ \hline
0.3 & (1,0,1) & (1,0,1) & (2,0,2) & (2,0,2) & (2,0,2) & (10,6,4) & (10,5,5)  \\ \hline
\end{tabular}
\end{center}
We can see that $m$ typically increases with $t_2$. At the same time, $d$ decreases when we increase the value of $\lambda_2$.

In the following table we fix $\lambda_2,t_2$, and let $\lambda_1,t_1$ vary.
More specifically, we take $n=k=10,\lambda_2=0.1$ and $t_2=6$. The optimal triple $(d,s,m)$ for different values of $\lambda_1$ and $t_1$ is recorded in the table.
\begin{center}
\begin{tabular}{| c | c | c | c | c | c | c | c |}
\hline
 \backslashbox{$\lambda_1$}{$t_1$} & 1 & 1.3 & 1.6 & 1.9 & 2.2 & 2.5 & 2.8 \\ \hline
0.5 & (10,8,2) & (10,8,2) & (3,1,2) & (3,1,2) & (3,1,2) & (2,0,2) & (2,0,2)  \\ \hline
0.6 & (10,8,2) & (10,8,2) & (3,1,2) & (3,1,2) & (3,1,2) & (3,1,2) & (2,0,2)  \\ \hline
0.7 & (10,8,2) & (3,1,2) & (3,1,2) & (3,1,2) & (3,1,2) & (3,1,2) & (3,1,2)  \\ \hline
0.8 & (10,8,2) & (4,1,3) & (4,1,3) & (3,1,2) & (3,1,2) & (3,1,2) & (3,1,2)  \\ \hline
0.9 & (10,7,3) & (4,1,3) & (4,1,3) & (4,1,3) & (3,1,2) & (3,1,2) & (3,1,2)  \\ \hline
1 & (10,7,3) & (4,1,3) & (4,1,3) & (4,1,3) & (4,1,3) & (3,1,2) & (3,1,2)  \\ \hline
\end{tabular}
\end{center}
We can see that for a fixed $\lambda_1$, $s$ decreases with $t_1$.

\appendices
\section{Converse proof of Theorem~\ref{thm:main}}\label{ap:conv}
Assume that $(d,s,m)$ is achievable, and let us prove \eqref{eq:obj}.
We first prove the following claim:
\begin{claim}\label{cl:ez}
For every $i\in[k]$, data subset $D_i$ must be assigned to at least $s+m$ workers.
\end{claim}
\begin{proof}
The proof goes by contradiction. Suppose for some $j\in[k]$, $D_j$ is assigned to $a<s+m$ workers. Without loss of generality we assume these $a$ workers are $W_1,W_2,\dots,W_a$. Now suppose that $W_1,W_2,\dots,W_s$ are the $s$ stragglers.
According to Definition~\ref{def:m}, we should be able to calculate $g_1+g_2+\dots+g_k$ from 
$f_i(g_{i_1},g_{i_2},\dots,g_{i_d}),i\in\{s+1,s+2,\dots,n\}$.

Observe that we can calculate the $l$-dimensional vector $g_j$ from the following set of vectors $\{g_i:i\in[k]\setminus\{j\}\}\cup
\{g_1+g_2+\dots+g_k\}$.
Therefore $g_j$ can also be calculated from $\{g_i:i\in[k]\setminus\{j\}\}\cup \{f_i(g_{i_1},g_{i_2},\dots,g_{i_d}):s+1\le i\le n\}$.
Since $D_j$ is only assigned to the first $a$ workers, the values of $\{f_i(g_{i_1},g_{i_2},\dots,g_{i_d}):a+1\le i\le n\}$ are determined by $\{g_i:i\in[k]\setminus\{j\}\}$.
As a result, $g_j$ can also be calculated from $\{g_i:i\in[k]\setminus\{j\}\}\cup \{f_i(g_{i_1},g_{i_2},\dots,g_{i_d}):s+1\le i\le a\}$.
Since $f_i,i\in[n]$ are all linear functions (see condition 3 of Definition~\ref{def:m}), we further deduce that $\{f_i(g_{i_1},g_{i_2},\dots,g_{i_d}):s+1\le i\le a\}$ must contain at least $l$ linear combinations of the coordinates of $g_j$.
On the other hand, each $f_i(g_{i_1},g_{i_2},\dots,g_{i_d})$ is a vector of dimension $l/m$, so $\{f_i(g_{i_1},g_{i_2},\dots,g_{i_d}):s+1\le i\le a\}$ contains at most $\frac{l}{m}(a-s)$ linear combinations of the coordinates of $g_j$.
As a result, we conclude that $\frac{l}{m}(a-s) \ge l$, i.e., $a\ge s+m$, which gives a contradiction. This completes the proof of Claim~\ref{cl:ez}.
\end{proof}
By Claim~\ref{cl:ez}, each data subset is assigned to at least $s+m$ workers, so in total there are at least $k(s+m)$ data subsets (counting with repetitions) assigned to all $n$ workers. Therefore each worker is assigned with at least $\frac{k}{n}(s+m)$ data subsets. Thus we have $d\ge \frac{k}{n}(s+m)$. This completes the proof of \eqref{eq:obj}.

\section{Differences between our results and the results in \cite{Dutta16}}\label{ap:dif}
Below we state the main result of \cite{Dutta16}. Note that we change their notation to comply with ours. 
\begin{theorem}[Theorem 1 in \cite{Dutta16}]\label{thm:c}
Given row vectors $a_1,a_2,\dots,a_m\in\mathbb{R}^{k'}$, there exists an $n\times k'$ matrix $Q$ such that any $(n-s)$ rows of $Q$ are sufficient to generate the row vectors $a_1,a_2,\dots,a_m$ and each row of $Q$ has at most $\frac{k'}{n}(s+m)$ nonzero entries, provided $n|k'$.
\end{theorem}

We first show that this theorem gives a coding scheme to achieve \eqref{eq:obj} with equality for the special case $m=1$, i.e., the case considered in \cite{Tandon17,Halbawi17,Raviv17}.
We set $m=1$ and $k'=k$ in Theorem~\ref{thm:c}, where $k$ is the number of data subsets in our problem. Moreover, we set $a_1$ to be the all one vector. In our gradient coding problem, we want to calculate $\sum_{i=1}^k g_i=a_i [\begin{array}{cccc} g_1 & g_2 & \dots & g_k \end{array}]^T$.
(Recall that $g_1,g_2,\dots,g_k$ are column vectors of dimension $l$.)
We denote the $i$th row of $Q$ in Theorem~\ref{thm:c} as $q_i$. 
We claim that each worker $W_i$ only needs to send the $l$-dimensional vector $q_i [\begin{array}{cccc} g_1 & g_2 & \dots & g_k \end{array}]^T$ to the master node.
Indeed, Theorem~\ref{thm:c} indicates that any $(n-s)$ rows of $Q$ suffice to generate $a_1$, so this coding scheme can tolerate any $s$ stragglers. Moreover, since the number of nonzero entries in each $q_i$ is at most $\frac{k}{n}(s+1)$, each worker only needs to be assigned with at most $\frac{k}{n}(s+1)$ data subsets. Therefore $d=\frac{k}{n}(s+1)$, achieving \eqref{eq:obj} with equality for the case $m=1$.

Next we argue that for $m>1$, Theorem~\ref{thm:c} cannot give coding schemes achieving \eqref{eq:obj} with equality.
For this case, in order to use Theorem~\ref{thm:c} for gradient coding, one needs to set $k'=km$. Moreover, for $u\in[m]$, we should set $a_u$ to be the $u$th row of the $m\times (mn)$ matrix $[\begin{array}{cccc} I_m & I_m & \dots & I_m \end{array}]$.
For every $v=0,1,\dots,l/m-1$ and $j\in[k]$, define an $m$-dimensional vector 
$$
y_v^{(j)}:=[\begin{array}{cccc}g_j(vm) & g_j(vm+1) & \dots & g_j(vm+m-1)\end{array}].
$$
For every $v=0,1,\dots,l/m-1$, define an $(mk)$-dimensional vector 
$$
z_v := [\begin{array}{cccc}
y_v^{(1)} & y_v^{(2)} & \dots & y_v^{(k)} \end{array}].
$$
Notice that the coordinates of the sum vector $g_1+g_2+\dots+g_n$ form the following set:
\begin{equation}\label{eq:fn}
\{a_u z_v^T:u\in[m],v\in\{0,1,2,\dots,l/m-1\}\}.
\end{equation}
Let each worker $W_i$ return the following $(l/m)$-dimensional vector 
$[\begin{array}{cccc} q_i z_0^T & q_i z_1^T & \dots & q_i z_{l/m-1}^T \end{array}]$.
Since any $(n-s)$ rows of $Q$ suffice to generate $a_1,a_2,\dots,a_m$, one can calculate the elements in the set \eqref{eq:fn} from the returned results of any $(n-s)$ workers and therefore recover the sum vector
$g_1+g_2+\dots+g_n$.
By Theorem~\ref{thm:c}, each $q_i$ has at most $\frac{mk}{n}(s+m)$ nonzero entries.
Now let us explain how the nonzero entries of $q_i$ correspond to the data subsets assigned to worker $W_i$, which is the reason why Theorem~\ref{thm:c} fails to give a gradient coding algorithm for $m>1$.
Let us further write $q_i=(q_{i,1},q_{i,2},\dots,q_{i,mk})$.
By definition of $z_v$, for any $u\in[m]$ and $j\in[k]$, if $q_{i,(j-1)m+u}\neq 0$, then worker $W_i$ needs
the value of $g_j(vm+u-1)$ to calculate $q_i z_v^T$, i.e., data subset $D_j$ should be assigned to $W_i$.
Thus we conclude that the number of data subsets assigned to $W_i$ is equal to
\begin{equation}\label{eq:qt}
|\{j:j\in[k],(q_{i,(j-1)m+1},q_{i,(j-1)m+2},\dots,q_{i,jm})\text{~is not a zero vector}\}|.
\end{equation}
In order to achieve \eqref{eq:obj}, we need the quantity in \eqref{eq:qt} to be no larger than 
$\frac{k}{n}(s+m)$ for all $i\in[n]$.
If this is the case, then each $q_i$ has at most $\frac{mk}{n}(s+m)$ nonzero entries, which is the condition in Theorem~\ref{thm:c}.
However, the condition in Theorem~\ref{thm:c} does not imply that the quantity in \eqref{eq:qt} is at most $\frac{k}{n}(s+m)$.
Thus the constraint in our problem is stronger than the constraint in Theorem~\ref{thm:c}, so the coding scheme in \cite{Dutta16} does not apply to our problem for the case $m>1$.

\section{Proof of Proposition~\ref{prop:p1}}\label{ap:p1}
The proposition follows immediately once we show that the optimal value of $d$ (we denote it $d^*$) can only be $1$ or $n$.
We prove by contradiction. Suppose that $1<d^*<n$, then by \eqref{eq:sk} we have
\begin{align*}
d^* t_1+\frac{d^*}{\lambda_1}(\sum_{i=0}^{n-d^*}\frac{1}{n-i}) &<
(d^*+1) t_1+\frac{d^*+1}{\lambda_1}(\sum_{i=0}^{n-d^*-1}\frac{1}{n-i}),\\
\text{and~} d^* t_1+\frac{d^*}{\lambda_1}(\sum_{i=0}^{n-d^*}\frac{1}{n-i}) &<
(d^*-1) t_1+\frac{d^*-1}{\lambda_1}(\sum_{i=0}^{n-d^*+1}\frac{1}{n-i}).
\end{align*}
Consequently,
$$
\sum_{i=d^*+1}^n \frac{1}{i}>1-\lambda_1 t_1, \text{~and~} 
\sum_{i=d^*}^n \frac{1}{i}<1-\lambda_1 t_1,
$$
which implies that $\sum_{i=d^*}^n \frac{1}{i}<\sum_{i=d^*+1}^n \frac{1}{i}$, but this is impossible. Therefore we conclude that $d^*$ can only be $1$ or $n$, and a simple comparison between these two gives the result in Proposition~\ref{prop:p1}.

\section{Proof of Proposition~\ref{prop:p2}}\label{ap:p2}
According to \eqref{eq:qnm}, we want to minimize the following function
$$
h(\alpha):= \frac{1}{\alpha n}(t_2-\frac{1}{\lambda_2}\log(1-\alpha)).
$$
Taking derivative of $h$, we have
$$
h'(\alpha)=\frac{1}{\alpha^2 n}\Big(\frac{1}{\lambda_2}\big(\frac{\alpha}{1-\alpha}+\log(1-\alpha) \big)-t_2 \Big).
$$
Define another function
$$
h_1(\alpha):= \frac{1}{\lambda_2}\big(\frac{\alpha}{1-\alpha}+\log(1-\alpha) \big)-t_2.
$$
Taking derivative of $h_1$, we have
$$
h_1'(\alpha)= \frac{1}{\lambda_2} \big(\frac{1}{(1-\alpha)^2} - \frac{1}{1-\alpha} \big).
$$
Clearly $h_1'(\alpha)> 0$ for all $0 < \alpha<1$.
Since $h_1(0)=-t_2<0$ and $h_1(1^-)=+\infty$, the equation $h_1(\alpha)=0$ has a unique solution $\alpha^*$ in the open interval $(0,1)$.
Moreover, since $h_1(\alpha)>h_1(\alpha^*)=0$ for all $\alpha>\alpha^*$ and
$h_1(\alpha)<h_1(\alpha^*)=0$ for all $\alpha<\alpha^*$, we also have
$h'(\alpha)>0$ for all $\alpha>\alpha^*$ and
$h'(\alpha)<0$ for all $\alpha<\alpha^*$. Consequently, $\alpha^*$ minimizes $h(\alpha)$, and this completes the proof of Proposition~\ref{prop:p2}.

\bibliographystyle{IEEEtran}
\bibliography{commgrad}

\end{document}